%% file: math.tex
\documentclass{article}

\usepackage{amsthm}
\usepackage{times}
\usepackage[margin=1.2in]{geometry}
\usepackage{amsmath}
\usepackage{algorithmic}
\usepackage{algorithm}
\usepackage{amsfonts}
\usepackage{comment}
\usepackage{graphicx}
\usepackage{hyperref}
\usepackage{color}		
\title{An optimal regret  algorithm  for  bandit convex optimization}
\author{Elad Hazan \thanks{Princeton University, Email: ehazan@cs.princeton.edu } \and Yuanzhi Li \thanks{Princeton University, Email: yuanzhil@cs.princeton.edu}}

\date{\today}

\include{defs}

\begin{document}

\maketitle

\begin{abstract}
We consider the problem of online convex optimization against an arbitrary adversary with bandit feedback, known as bandit convex optimization. We give the first $\tilde{O}(\sqrt{T})$-regret algorithm for this setting  based on a novel application of the ellipsoid method to  online learning. This bound is known to be  tight up to logarithmic factors.   Our analysis introduces new tools in  discrete convex geometry. 
\end{abstract}

\newcommand{\MVEE}{\textsf{MVEE}}

\input{introduction}

\section{Preliminaries} \label{sec:preliminaries}

The setting of bandit convex optimization (BCO) is a repeated game between an online learner and an adversary (see e.g.  \cite{HazanBook} chapter 6). Iteratively, the learner makes a decision which is a point in a convex decision set, which is a subset of Euclidean space $x_t \in \K \subseteq \reals^d$. Meanwhile, the adversary responds with an arbitrary Lipschitz convex loss function $f_t: \K \mapsto \reals$. The only feedback available to the learner is the loss, $f_t(x_t) \in \reals$, and her goal is to minimize regret, defined as
$$ \R_T = \sum_t f_t(x_t) - \min_{x^* \in \K} \sum_t f_t(x^*) $$

Let $\K \subseteq \reals^d$ be a convex compact and closed subset in Euclidean space. 
We denote by $\ellipsoid_\K$ the minimal volume enclosing ellipsoid (\MVEE) in $\K$, also known as the John ellipsoid \cite{John48,Ball97}. For simplicity, assume that $\ellipsoid_\K$ is centered at zero. 

Given an ellipsoid $\ellipsoid = \{\sum_i \alpha_i v_i \,:\, \sum_i \alpha_i^2 \leq 1\}$, we shall use the notation $\|x\|_{\ellipsoid} \equiv \sqrt{x^\top (VV^\top)^{-1} x}$ to denote the (Minkowski) semi-norm defined by the ellipsoid, where $V$ is the matrix with the vectors $v_i$'s as columns.

John's theorem says that if we shrink \MVEE \ of $\K$ by a factor of $1/d$, then it will be inside $\set{K}$. For connivence, we denote by $\|\cdot\|_\K$ the norm according to $\frac{1}{d}\ellipsoid_\K$, which is the matrix norm corresponding to the (shrinked by factor $1/d$) $\MVEE$ ellipsoid of $\K$ . To be specific, Let $\ellipsoid$ be the \MVEE \ of $\K$, 
$$\| x\|_{\K} = d \| x \|_{\ellipsoid} = \| x \|_{\frac{1}{d}\ellipsoid}  $$

We use $ d \| x \|_{\ellipsoid}$ inside of $ \| x \|_{\ellipsoid}$ merely to insure $\forall x \notin \set{K}$, $\| x\|_{\K}  \ge 1$, which simplifies our expression.

\paragraph{Enclosing  box.} Denote by $C_\K$ the bounding box of the ellipsoid $\ellipsoid_\K$, which is obtained by the box with axis parallel to the eigenpoles of $\ellipsoid_\K$. 
The containing  box $C_\K$ can be computed by first computing $\ellipsoid_\K$, then the diagonal  transformation of this ellipsoid into a ball, computing the minimal enclosing cube of this ball, and performing the inverse diagonal transformation into a box.

\begin{defn}[Minkowski Distance of a convex set]
Given a convex set $\set{K} \subset \reals^d$ and $x \in \reals^d$, the Minkowski distance $\distanceRatio(x, \set{K})$ is defined as
$$\distanceRatio(x, \set{K}) = || x  - x_0 ||_{\K - x_0}$$
Where $x_0$ is the center of the $\MVEE$ of $\K$. $\K - x_0$ denotes shifting $\K$ by $-x_0$ (so its \MVEE \ is centered at zero)
\end{defn}

\begin{defn}[Scaled set]
For $\beta > 0 $, define $\centerShrink \set{K}$ as the scaled set \footnote{According to our definition of $\distanceRatio$, $1 \K \subseteq \K \subseteq d\K$}
$$ \centerShrink \set{K} = \{ y \mid \distanceRatio(y, \set{K})  \le \centerShrink \} $$ 
\end{defn}

Henceforth we will require a discrete representation of convex sets, which we call grids, as  constructed in Algorithm \ref{alg:grid}. 

\begin{algorithm}[h!]
\caption{ construct grid \label{alg:grid} } 
\begin{algorithmic}[1]
\STATE Input: convex set ${\K} \in \reals^d$, resolution $\gridScale$. 
\STATE Compute the \MVEE \ $\ellipsoid'$ of ${\convexSet}$. Let $\ellipsoid = \frac{1}{d} \ellipsoid'$
\STATE Let $A$ be the (unique) linear transformation such that $A(\ellipsoid) = \ball_{\gridScale}(0)$ (unit ball of radius $ \gridScale$ centered at 0). 
\STATE Let $\set{Z}^d = \{ (x_1,...,x_d)  , \ x_i \in Z \}$ be $d$-dimensional integer lattice.
\STATE Output: $\grid = A^{-1}(\set{Z}) \cap { \convexSet}  $.
\end{algorithmic}
\end{algorithm}

\begin{claim} For every $\K \in \reals^d$, $\grid = \grid(\K, \gridScale)$ contains at most $(2d \gridScale)^d$ many  points 
\end{claim}

\begin{lem}[Property of the grid] \label{property:grid}

Let $\hconvexSet' \subseteq \convexSet \subseteq \reals^d$ \footnote{We will apply the lemma to $\hconvexSet' $ being our working Ellipsoid and $\K$ being the original input convex set} be two convex sets. For every $\centerShrink ,\extendRatio$ such that  $\centerShrink > \extendRatio > 1$, $\centerShrink > d$, for every $\gridScale \ge  2( \extendRatio+1)    \centerShrink^2 \sqrt{d}$ such that the following holds.  Let $\grid = \grid(\centerShrink \hconvexSet' \cap \convexSet, \gridScale)$, then we have:
\begin{enumerate}
\item \label{property1}
For every $x \in \hconvexSet'$: $\exists x_g \in \grid $ such that $x_g + \extendRatio (x_g - x) \in  \frac{1}{2 \centerShrink}{\hconvexSet'}$
\item
For every $x \notin \hconvexSet'$, $x \in \hconvexSet$: $\exists x_g \in \grid $ such that $x_g +  \frac{\extendRatio}{\distanceRatio(x,  \hconvexSet')} (x_g - x) \in \frac{1}{2 \centerShrink}{\hconvexSet'}$
\end{enumerate}
\end{lem}

\begin{center}
\begin{figure}
\begin{center}
\caption{The property of the grid}
\includegraphics[width=0.6\textwidth]{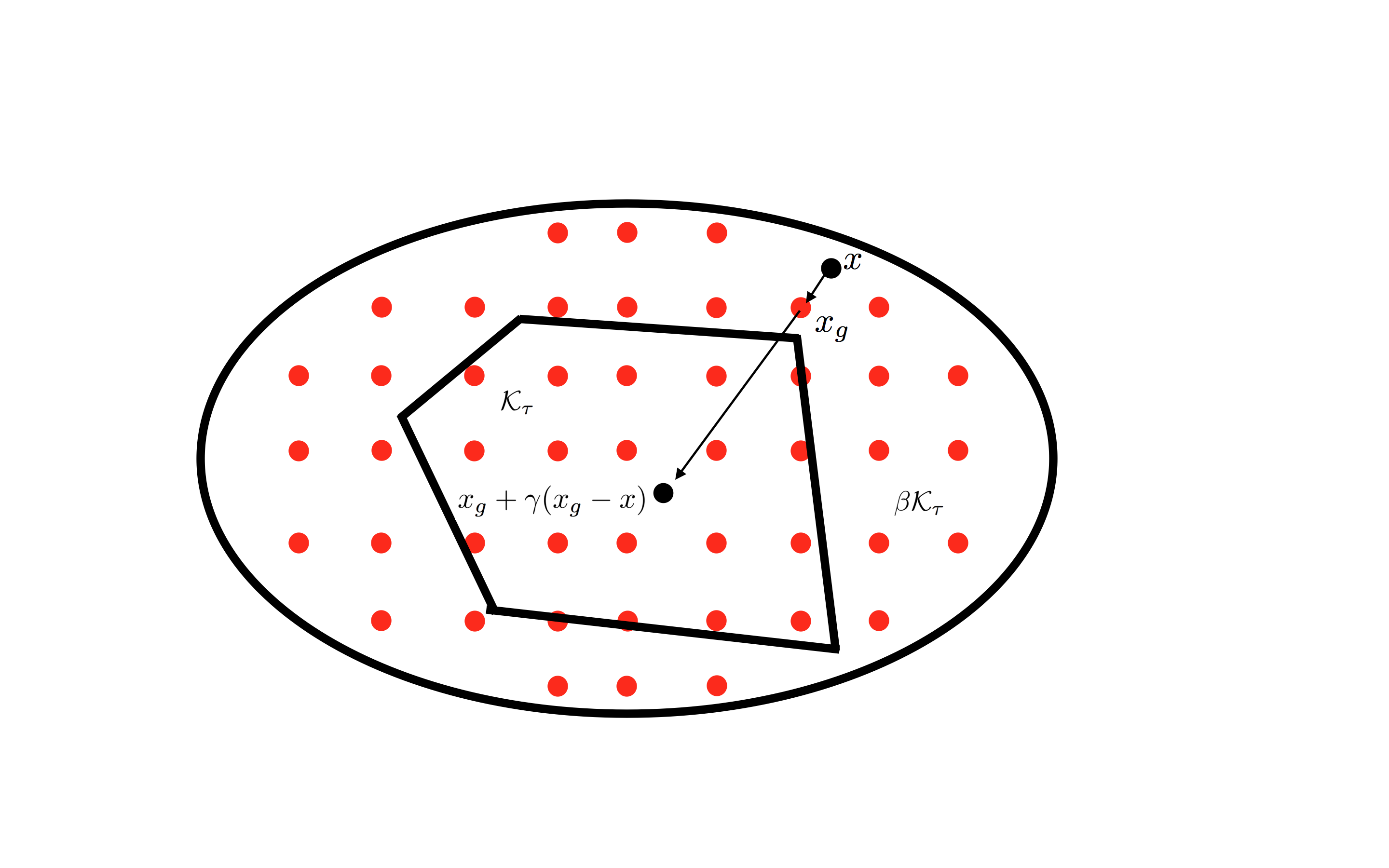}
\end{center}
\end{figure}
\end{center}

\begin{proof}[Proof of Lemma \ref{property:grid}]

Since $\beta > d$, by John's theorem, $\K' \subset \beta  \K'$. Moreover, since we only interested in the distance ratio, we can assume that the \MVEE \ $\ellipsoid'$ of $\centerShrink \hconvexSet' \cap \convexSet$ is the ball centered at $0$ of radius $d \gridScale$, and $\grid$ are all the integer points intersected with $\centerShrink \hconvexSet' \cap \convexSet$. Let $\ellipsoid = \frac{1}{d} \ellipsoid' = \ball_{\alpha}(0)$, by John's Theorem, we know that $\ellipsoid \subseteq \centerShrink \hconvexSet' \cap \convexSet \subseteq d \ellipsoid$.

(a). For every $x \in \hconvexSet'$, consider point $z = \frac{\extendRatio }{\extendRatio + 1} x$. Since $ \ellipsoid = \ball_{\alpha}(0) \subseteq  \centerShrink \hconvexSet' $, we know that $ \ball_{\frac{\alpha}{\beta}}(0) \subseteq \K'$. Therefore, $\ball_{\frac{\alpha}{\extendRatio \beta}}(z) \subseteq \K'$, which implies when $\alpha \ge \extendRatio \beta\sqrt{d}$, we can find $x_g \in \grid$ such that $\|x_g - z \|_2 \le \sqrt{d}$. Therefore, 
\begin{eqnarray*}
\|x_g + \extendRatio (x_g - x) \|_2 & =& \|\left[z+ \extendRatio (z - x) \right] + \left[ x_g - z+\extendRatio ( x_g - z) \right]\|_2 \\
& =&  ||x_g - z+\extendRatio (x_g - z)||_2 \quad \quad   \mbox{ since $z+ \extendRatio (z - x) = 0$}  \\
& = &(\gamma +1)  \| x_g -z \|_2 \leq  ( \extendRatio+1) \sqrt{d}    \quad \quad   \mbox{ by $\|x_g - z\| \leq \sqrt{d}$}
\end{eqnarray*}
Moreover, $\frac{1}{2\beta} \hconvexSet \supseteq \frac{1}{2\beta^2}  \ellipsoid = \frac{1}{2\beta^2}  \ball_\alpha (0)$ contains all points with norm $\frac{\alpha}{2\beta^2}$, and in particular it contains $x_g + \extendRatio (x_g -x) $ when $\gridScale \ge 2( \extendRatio+1)  \centerShrink^2 \sqrt{d}$.

(b). For every $x \notin \hconvexSet' $ but $x \in \K$, take $z =  \frac{\extendRatio}{\distanceRatio(x,  \hconvexSet') + \extendRatio} x$. When $\beta > 2\gamma$, we know that $z \in \frac{1}{2}\beta \K'$. With same idea as (a), we can also conclude that $$\ball_{\frac{\distanceRatio(x,  \hconvexSet') }{\distanceRatio(x,  \hconvexSet') + \extendRatio}\frac{\alpha}{\beta^2} } (z) \subseteq \centerShrink \hconvexSet' \cap \convexSet$$ 

Since $\distanceRatio(x,  \hconvexSet') \ge 1$ for $x \notin \K'$, we can find $x_g \in \grid$ be such that $\|x_g - z\|_2 \le \sqrt{d}$ when $\gridScale \ge (\extendRatio + 1)\centerShrink^2\sqrt{d}$.  Therefore, 

\begin{eqnarray*}
&& \left \| x_g + \frac{\extendRatio}{\distanceRatio(x,  \hconvexSet)}  (x_g - x)  \right\|_2 \\
& =& \left\|\left[z+ \frac{\extendRatio}{\distanceRatio(x,  \hconvexSet)}  (z - x) \right] + \left[ x_g - z+\frac{\extendRatio}{\distanceRatio(x,  \hconvexSet)}  (x_g - z) \right]  \right\|_2 \\
& =& \left\| x_g - z+\frac{\extendRatio}{\distanceRatio(x,  \hconvexSet)}  (x_g - z) \right \|_2  \quad \quad  \mbox{ since $z+ \frac{\extendRatio}{\distanceRatio(x,  \hconvexSet)}  (z - x) = 0$}  \\
& = & \left(1 + \frac{\extendRatio}{\distanceRatio(x,  \hconvexSet)} \right) \|x_g - z\|_2 \leq  \left( \frac{\extendRatio}{\distanceRatio(x,  \hconvexSet)} + 1\right)  \sqrt{d} \\
& \leq &  (1 + {\extendRatio})  \sqrt{d} \quad \quad  \mbox{ since $\distanceRatio(x,  \hconvexSet) \ge 1$} 
\end{eqnarray*}

As before, this  implies that when $\gridScale \ge  2 (\extendRatio  + 1)  \centerShrink^2 \sqrt{d}$, it holds that $x_g + \frac{\extendRatio}{\distanceRatio(x,  \hconvexSet)}  (x_g - x) \in \frac{1}{2 \centerShrink }{\hconvexSet}$.

\end{proof}

\subsection{Non-stochastic bandit algorithms}  \label{sec:lradefinitions}

Define the following
$$(p_t, \vLRA_t, \varLRA_t ) \leftarrow \LRA(S,  \{p_{t-1},f_{1:t-1} \} )$$

$p_t$: A probability distribution over the discrete set $S$ 

$\vLRA_t$: Estimation of the values of $F^{t} = \sum_{i =1}^t f_i$ on $S$.

$\varLRA_t$: Variance, such that for every $x \in S$, $\vLRA_t(x) - \varLRA_t(x) \le F_{t}(x) \le \vLRA_t(x) + \varLRA_t(x)$.

For $x_t$ picking according to distribution $p_t$, define the regret of $\LRA$ as: 
$$\R_T  = \sum_t f_t(x_t)  -  \min_x \left\{\sum_t f_t(x) \right\} $$

The following theorem was essentially established in \cite{Auer2003} (although the original version was stated for gains instead of losses, and had known horizon parameter), for the algorithm called EXP3.P, which is given in Appendix \ref{sec:exp3} for completeness:
\begin{theorem}[\cite{Auer2003}]
Algorithm EXP3.P over $N$ arms guarantees that with probability at least $1-\delta$, 
$$\R_T  = \sum_t f_t(x_t)  -  \min_x \left\{\sum_t f_t(x) \right\} \le  8\sqrt{ T N \log \frac{TN}{\delta} } $$ 
\end{theorem}

\input{convex-regression}

\section{Geometry of discrete convex function}  \label{sec:geometry}

\subsection{Lower convex envelopes of continuous and discrete convex functions}  

Bandit algorithms generate a discrete set of evaluations, which we have to turn into convex functions. The technical definitions that allow this are called lower convex envelopes (LCE), which we define below. First, for continuous but non-convex function $f$, we can define the LCE denoted $\FLCE(f)$ as the maximal convex function that bounds $f$ from below, or formally,

\begin{defn}[Simple Lower Convex Envelope]
Given a function $f: \set{K} \to \reals$ (not necessarily convex) where $\set{K} \subset \reals^d$, the simple lower convex envelope $\FSLCE = \SLCE(f) : \set{K} \to \reals$ is a convex function defined as:
$$\FSLCE(x) = \min_{} \left\{ \sum_{i = 1}^{s} \lambda_i f(y_i) \ \right| \ \left.  \exists s \in \mathbb{N}^*, y_1, ..., y_{s} \in \set{K}: \exists (\lambda_1, ..., \lambda_{s}) \in \Delta^{s}, x = \sum_i \lambda_i y_i \right\}$$
\end{defn}

It can be seen that $\FSLCE$ is always convex, by showing for every $x,y \in \K$ that $f(\frac{1}{2}x + \frac{1}{2} y) \leq \frac{1}{2} f(x) + \frac{1}{2} f(y) $, which follows from the definition. Further,  for a convex function, $\FSLCE(f)=f$, since for a convex function any convex combination of points satisfy $f(\sum_i \lambda_i y_i) \leq \sum_i \lambda_i f(y_i)$, and the minimum in the definition is realized at the point $x$ itself.

For a discrete function, the SLCE is defined to be the SLCE of the piecewise linear continuation. 

We will henceforth need a significant generalization of this notion, both for the setting above, and for the setting in which  the discrete function is given as a random variable - on each point in the grid we have a value estimation and variance estimate.  We first define the minimal extension, and then the SLCE of this minimal extension.

\begin{defn}[Random Discrete Function]\label{defn:RDF}
A Random Discrete Function (RDF), denoted $(X,v,\sigma)$, is a mapping $f: X\to \reals^2$ on a  discrete domain  $ X = \{ x_1, ..., x_k \} \subseteq \K \subseteq \reals^d$, and range of  values and variances denoted $\{v(x),\sigma(x) , x \in X\}$ such that $f(x_i) = (v(x_i), \sigma(x_i))$.
\end{defn}

\begin{defn}[Minimal Extension of a Random Discrete Function]\label{defn:min_extension}
Given a RDF $(X,v,\sigma)$, we define $\tf^i_{\min}(X,v,\sigma): \K \to \reals$ as $$\tf^i_{\min}(x) = \min_{h \in \reals^d: \forall x_j, \langle h, x_j - x_i \rangle \le  \vLRA(x_j) +  \varLRA(x_j)   - [\vLRA(x_i) -  \varLRA(x_i) ] }  \left\{\langle h, x - x_i \rangle +  [\vLRA(x_i) -  \varLRA(x_i) ] \right\}$$
The minimal extension $\tf_{\min}(X,v,\sigma) $ is now  defined as 
$$\tf_{\min} (x) = \max_{i \in [k]}\tf^i_{\min}(x) $$
\end{defn}

We can now define the LCE of a discrete random function
\begin{defn}[Lower convex envelope of a random discrete function] \label{defn:lce}
Given a RDF $(X,v,\sigma)$ over domain $X =  \grid  \subseteq \K \subseteq \reals^d$, for the grid for $\K$ as constructed in Algorithm \ref{alg:grid}, its lower convex envelope is defined to be 
$$ \FLCE(X,v,\sigma) =  \FSLCE(\tf_{\min}(X,v,\sigma)) $$
\end{defn}

\begin{center}
\begin{figure}
\begin{center}
\caption{The minimal extension and LCE of a discrete function}
\includegraphics[width=0.6\textwidth]{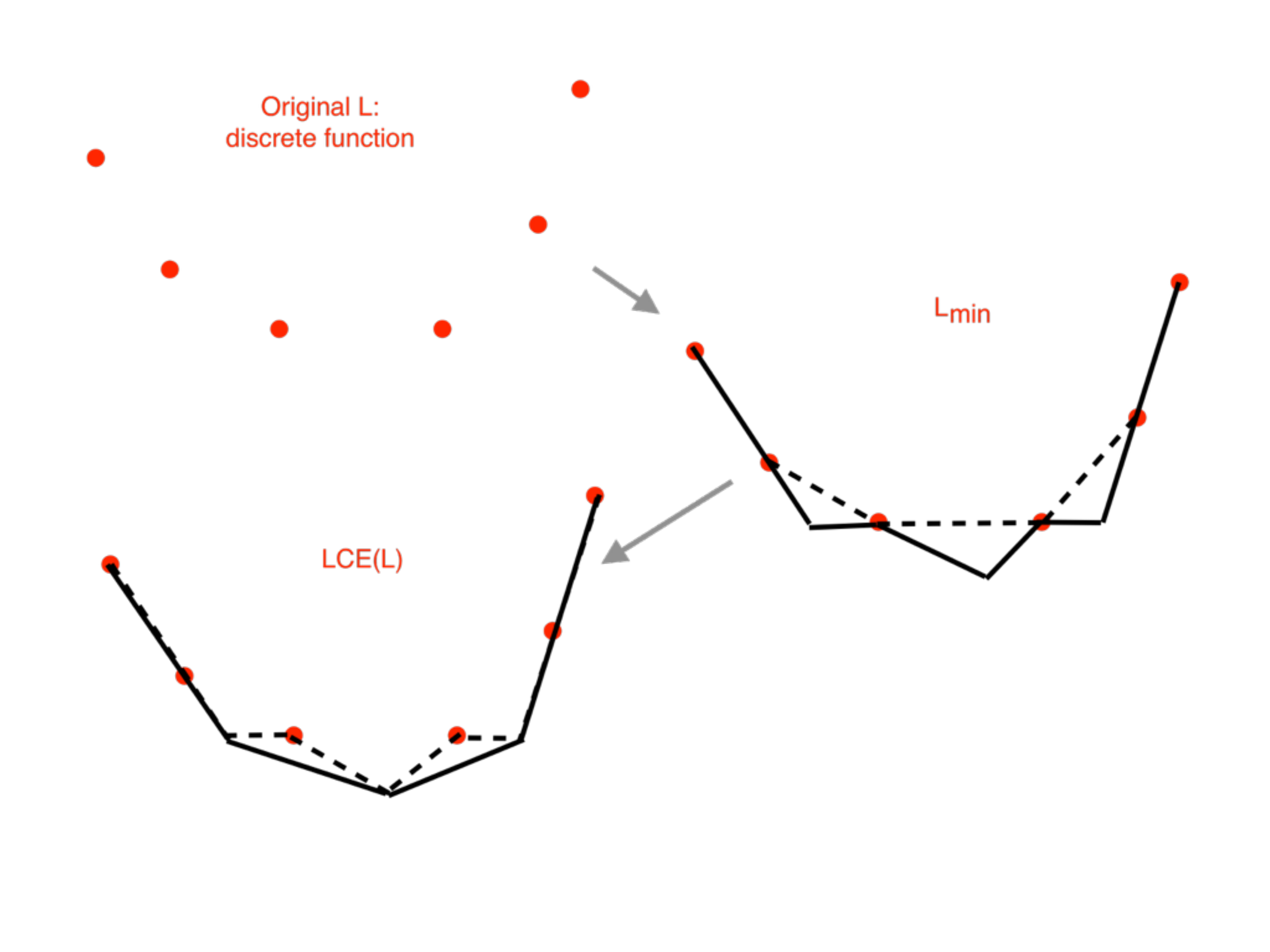}
\end{center}
\end{figure}
\end{center}

We now address the question of computation of an LCE of a discrete function, or how to provide oracle access to the LCE efficiently. The following theorem and algorithm establish the computational part of this section, whose proof is deferred to the appendix.

\begin{algorithm}[h!]
\caption{ Fit-LCE } 
\begin{algorithmic}[1]
\STATE Input: RDF $(X,v,\sigma)$, and a convex set $\hconvexSet$ where $X \subseteq \hconvexSet$.
\STATE (minimal extension): Compute the minimal extension  $\tf_{\min}(X,v,\sigma): \boundingBox_{\hconvexSet} \to \reals$ (see Section \ref{sec:preliminaries} for definition of the bounding box $\boundingBox$)
\STATE 
(LCE) Compute and return $\FLCE = \SLCE(\tf_{\min})$. 
\end{algorithmic}
\label{alg:fit_LCE}
\end{algorithm}

\begin{thm}[LCE computation] \label{thm:LCE} Given a discrete random function over $k$ points $\{x_1, ..., x_k\}$ in a polytope $\K \subseteq \reals^d$ defined by $N = \poly(d) $ halfspaces, with confidence intervals $[v(x_i) - \varLRA(x_i), \vLRA(x_i) + \varLRA(x_i)]$ for each point $x_i$, then for every $x \in \K$, the value $\FLCE(x)$  can be computed in time $O\left( k^{d^2}  \right)$
\end{thm}

To prove the running time of LCE computation, we need the following  Lemma:
\begin{lem}[LCE  properties] \label{lem:LCE} The lower convex envelope (\emph{LCE}) has the follow properties:

\begin{enumerate}
\item $\tf_{\min}$ is a piece-wise linear function with $k^{O(d^2)}$ different regions, each region is a polytope with $d+1$ vertices. We denote all the vertices of all regions as $v_1, ..., v_{n}$ where $n = k^{O(d^2)}$, where each $v_i$ and its value $\tf_{\min}(v_i)$ are computable in time $k^{O(d^2)}$.
\item $$\FLCE(x) = \min \left\{ \sum_{i \in [n]} \lambda_i \tf_{\min}(v_i) \right| \left. \sum_{i \in [n]} \lambda_i v_i = x, (\lambda_1, ..., \lambda_n) \in \Delta^n \right\}$$
\end{enumerate}
\end{lem}

\begin{proof}
%
%

Recall the definition of $\tf^i_{\min}: \hconvexSet \to \reals$ as $$\tf^i_{\min}(x) = \min_{h \in \reals^d: \forall x_j, \langle h, x_j - x_i \rangle \le  \vLRA(x_j) +  \varLRA(x_j)   - [\vLRA(x_i) -  \varLRA(x_i) ] }  \left\{\langle h, x - x_i \rangle +  [\vLRA(x_i) -  \varLRA(x_i) ] \right\}$$
The vector $h$ in the above expression is the result of a linear program. Therefore, it belongs to the vertex set of the polyhedral set given by the inequalities 
$\langle h, x_j - x_i \rangle \le  \vLRA(x_j) +  \varLRA(x_j)   - [\vLRA(x_i) -  \varLRA(x_i) ] $, or the objective is unbounded, a case which we can ignore since $\tf_{\min}$ is finite. The number of vertices of a polyhedral set in $\reals^d$ defined by $k$ hyperplanes is bounded by  ${k \choose d} \leq k^d$.  

Thus, $\tf^i_{\min}$ is the minimal of a finite set of linear functions at any point in space. This implies that it is a piecewise linear  function with at most $k^d$ regions.  More generally,  the minimum of $s$ linear functions is a piece-wise linear function of at most $s$ regions, as we now prove:

\begin{lemma}
The minimum (or maximum) of $s$ linear functions is  a piecewise linear function with at most $s$ regions. 
\end{lemma}
\begin{proof}
Let $f(x) = \min_{i \in [s] } f_i( x)$ for linear functions $\{f_i\}$, the proof for $\max_{i \in [s]} f_i(x)$ is analoguous. Consider the sets  $S_i = \{ x \mid f(x) = f_i (x) \}$, inside which $f = f_i$ is linear. It suffices  to show that each $S_i$ is a convex set, and thus each $S_i$ is a polyhedral region with at most $s$ faces. Now suppose $x_1, x_2 \in S_i$, we want to argue that $x_3 = \frac{x_1 + x_2}{2} \in S_i$: Observe that for every $j$, $f_j(x_3) = \frac{f_j(x_1 ) + f_j(x_2)}{2}$ (this is because $f_j$ is linear). If there is a $j$ such that $f_j(x_3) < f_i(x_3)$, then either $f_j(x_1 ) < f_i(x_1)$ or $f_j(x_2 ) < f_i(x_2)$, contradict to the fact that $x_1, x_2 \in S_i$. 
\end{proof}

Next we consider 
$$\tf_{\min} (x) = \max_{i \in [k]}\tf^i_{\min}(x) $$

Recall that each $\tf_{\min}^i$ is piecewise linear with $s = k^d$ regions who are determined by at most $s$ hyperplanes.   Consider regions in which all these functions are jointly linear, we would like to bound the number of such regions.  These regions are created by the hyperplanes that create the regions  of the functions $\tf_{\min}^i$,  a total of at most $k s$ hyperplanes, plus $N$ hyperplanes of the bounding polytope $\K$.  The number of regions these hyperplanes create is at most $(N + ks)^2$ \cite{abebe}.  In each such region, the functions $\tf_{\min}^i$ are linear, and according to the previous lemma there are at most $k$ sub-regions, giving a total of $k \times (N + ks)^2 \leq k N^2 + k^{3d}   $ polyhedral regions within which the function $\tf_{\min}$ is linear. 

The vertices of these regions can be computed by taking all $d$ intersections of the $(N+ks)^2$ hyperplanes and solving a system of $d$ equations, in overall time $(N+ks)^{2d} = k^{O(d^2)}$.

2. By definition of $\FLCE$, there exists points $p_1, ..., p_m \in \hconvexSet$ and $(\lambda_1, ..., \lambda_m) \in \Delta^m$ such that 
\begin{eqnarray}\label{eq:LCE}\FLCE(x) = \sum_{i \in [m]} \lambda_i \tf_{\min}(p_i), \sum_{i \in [m]} \lambda_i p_i = x 
\end{eqnarray}

By part 1, $\tf_{\min}$ is a piece-wise linear function, we know that for every $i \in [m]$, there exists $d + 1$ vertices $v_{i_1}, ..., v_{i_{d+1}}$ such that there exists $(\lambda_{i_1}, ..., \lambda_{i_{d + 1}}) \in \Delta^{d + 1}$ with $\sum_{j \in [d + 1]} \lambda_{i_j} \tf_{\min} (v_{i_j}) = \tf_{\min}(p_i), \sum_{j \in [d + 1]} \lambda_{i_j}v_{i_j} = p_i$. Put it into Equation \ref{eq:LCE} we get the result. 
\end{proof}

Having Lemma \ref{lem:LCE}, we can calculate $\FLCE$ by first finding vertices $v_1, ..., v_n$ and then solve an LP on $\lambda_i$. The algorithm runs in time $k^{O(d^2)}$

\section{The discretization lemma}  \label{sec:discretization}

The tools for discrete convex geometry developed in the previous section, and in particular the lower convex envelope, are culminated in the discretization lemma that shows consistency of the LCE for discrete random functions which we prove in this section. 

Informally, the discretization lemma asserts that for any value of a given RDF, the LCE has a point with value at least as large not too far away.  Convexity is crucial for this lemma to be true at all, as demonstrated in Figure \ref{fig:nolce}.

\begin{center}
\begin{figure}[h!]
\begin{center}
\caption{The LCE cannot capture global behavior for non-convex functions. \label{fig:nolce}}
\includegraphics[width=0.5\textwidth]{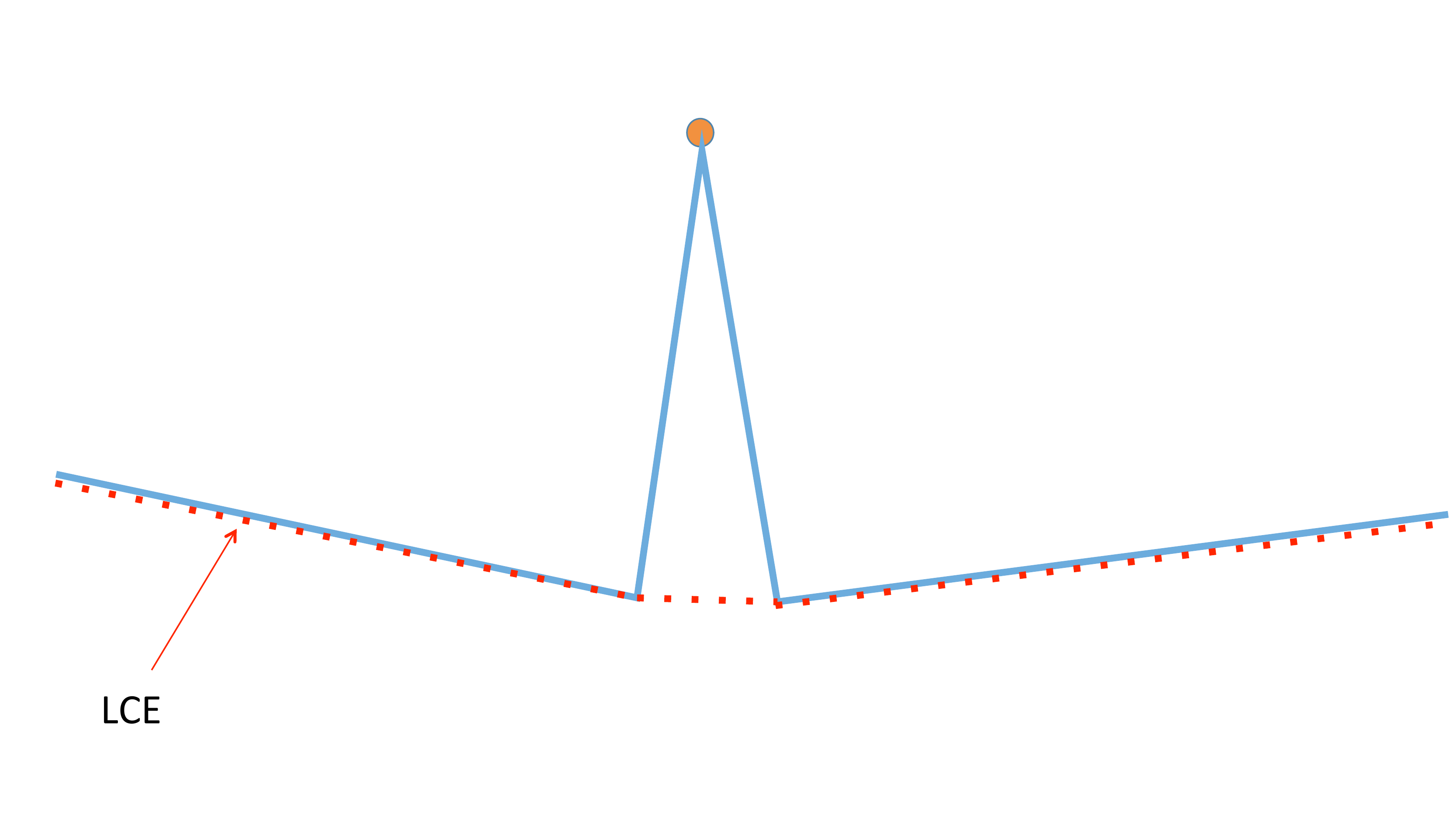}
\end{center}
\end{figure}
\end{center}

We now turn to a precise statement of this lemma and its proof:

\begin{lem}[Discretization]\label{lem:geometry}. 
Let $(X,v,\sigma)$ be a RDF on  $X = \set{Z}^d \cap \convexSet$ such that $v, \sigma$ are non-negative, moreover, for all $x \in X, v(x) - (8d^2 + 1)\sigma(x) \ge 0$. Assume further that there exists a convex function $F: \reals^d \mapsto \reals$ such that for all $x \in X$, $F(x) \in [v(x) - \sigma(x), v(x) + \sigma(x)]$.  Let $\convexSet' =  \boundingBox_{\K}$ be the enclosing bounding box for $\K$ such that $\ball_{2^{4d^2}} (0)\subseteq \frac{4}{d^2}\convexSet' \subseteq \convexSet \subseteq \convexSet' $ \footnote{John's theorem  implies $\frac{1}{d^{3/2}}\convexSet' \subseteq \convexSet \subseteq \convexSet'$ for any convex body $\K$}. Define  $\FLCE = \LCE(X,v,\sigma):  \convexSet' \to \reals$ as in Definition \ref{defn:lce}.

Then there exists a value $\radiusSearch = 2^{3d^2}$ such that for every $y \in \frac{1}{4}\convexSet$ with $\ball_{\radiusSearch}(y ) \subseteq \convexSet$, there exists a point $y' \in \ball_{\radiusSearch}(y )$ with $\FLCE(y') \ge \frac{1}{2} F(y)$.
\end{lem}

\subsection{Proof intuition in one dimension}

The discretization lemma is the main technical challenge of our result, and as such we first give intuition for the proof for one dimension,  for readability purposes only, and for the special case that the input DRF is actually a deterministic function (i.e. all variances are zero, and $v(x_i) = F(x_i)$ for a convex function $F$. The full proof is deferred to the appendix.

\begin{proof}

Please refer to Figure \ref{fig:1d} for an illustration. 
\begin{center}
\begin{figure}[h!]
\begin{center}
\caption{Discretization lemma in 1-d \label{fig:1d}}
\includegraphics[width=0.6\textwidth]{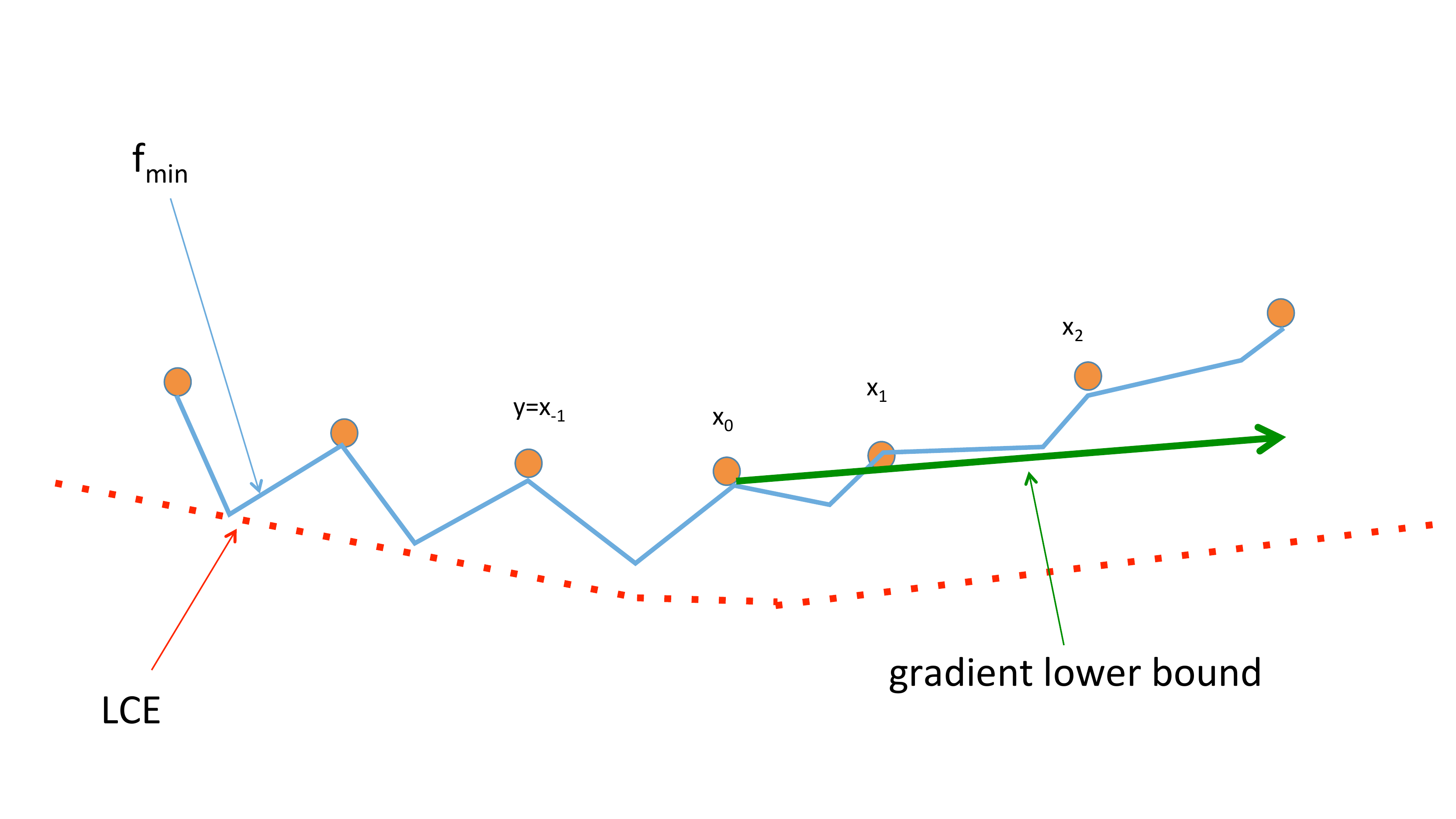}
\end{center}
\end{figure}
\end{center}

Assume w.l.o.g. that $y \in \set{Z}$, otherwise take the nearest point. 
Assume w.l.o.g that $f'(y) > 0 $, and thus all points $ x>y$ have value larger than $F(y)$. Consider the discrete points $\{y = x_{-1},x_0 ,x_1,...., \}$, and the value of $\tf_{\min}$ on these integer points, which by definition has to be equal to $F$, and thus larger than $F(y)$. 

Since $F$ is increasing in the positive direction, we have that $\tf_{\min}(x_0) \leq \tf_{\min}(x_1)$, and by the definition of $\tf_{\min}$, the gradient from $x_0$ to $x_1$, implies that 
$$ \forall z \geq x_1,   \ \tf_{\min}(z) \geq \tf_{\min}(x_0)$$

In the open interval $[x_2,\infty)$, the value of the $\LCE$ is by definition a convex combination of values $\tf_{\min}(x)$ only for points in the range $x \in[x_1,\infty) $. Thus, the function $\FLCE$ obtains a value larger than $\tf_{\min}(x_1) \geq F(y)$ on all points within this range, which is within a distance of two from $y$. 
\end{proof}

The proof of the Discretization Lemma requires the following lemmas:

\begin{lem}[Convex cover]\label{lem:convex_cover} For every $k \in \mathbb{N}^*$, $r \in \reals^*$,
if $k$ convex sets $\set{S}_1, ..., \set{S}_k$ covers a ball in $\reals^d$ of radius $r$, then there exists a set $\set{S}_i$ that contains a ball of radius $\frac{r}{k d^d}$.
\end{lem}

\begin{proof}[Proof of Lemma \ref{lem:convex_cover}]
Consider the maximum volume contained Ellipsoid $\set{E}_i$ of $\set{S}_i \cap \ball_r(0)$, we know that the volume of $\set{E}_i$ is at least $1/d^d$ the volume of $\set{S}_i \cap \ball_r(0)$. Now, since $\set{S}_1, ..., \set{S}_k$ covers $\ball_r(0)$, there exists a set $\set{S}_i \cap \ball_r(0)$ of volume at least $1/k$ fraction of the volume of $\ball_r(0)$. Which implies that $\set{E}_i$ has volume at least $1/(kd^d)$ of $\ball_r(0)$, note that $\set{E}_i \subseteq \ball_r(0)$, therefore, it contains a ball of radius $\frac{r}{k d^d}$.
\end{proof}

\begin{lem}[Approximation of polytope with integer points]  \label{lem:approximation}Suppose a polytope $\set{P}_o = \conv \{v_1, ..., v_{d + 1}  \} \subseteq \reals^d$ contains $\ball_{4d^8}(0)$, then there exists $d+1$ \emph{integer points} $g_1, ..., g_{d + 1} \in \frac{2}{d^2} \set{P}_o$ such that:
\begin{enumerate}
\item Let $(\lambda_1, ..., \lambda_{d + 1}) \in \Delta^{d + 1}$ be the coefficient such that $\sum_i \lambda_i v_i = 0$, then there exists $(\lambda_1', ..., \lambda_{d + 1}') \in \Delta^{d + 1}$ such that  $\sum_i \lambda_i' g_i = 0$. Moreover, $\frac{1}{2} \lambda_i' \le \lambda_i \le 2 \lambda_i'$
\item For every $i \in [d + 1]$,  there exists $\{ \lambda^i_j \in \Delta^{d+1} \}$ such that $\lambda_i^i  \ge \frac{1}{2d^2}$ and 
$$g_i = \lambda^i_i v_i + \sum_{j \not= i} \lambda_j^i g_j $$
\end{enumerate}


\end{lem}

\begin{center}
\begin{figure}[h!]
\begin{center}
\caption{Approximation Lemma}
\includegraphics[width=0.5\textwidth]{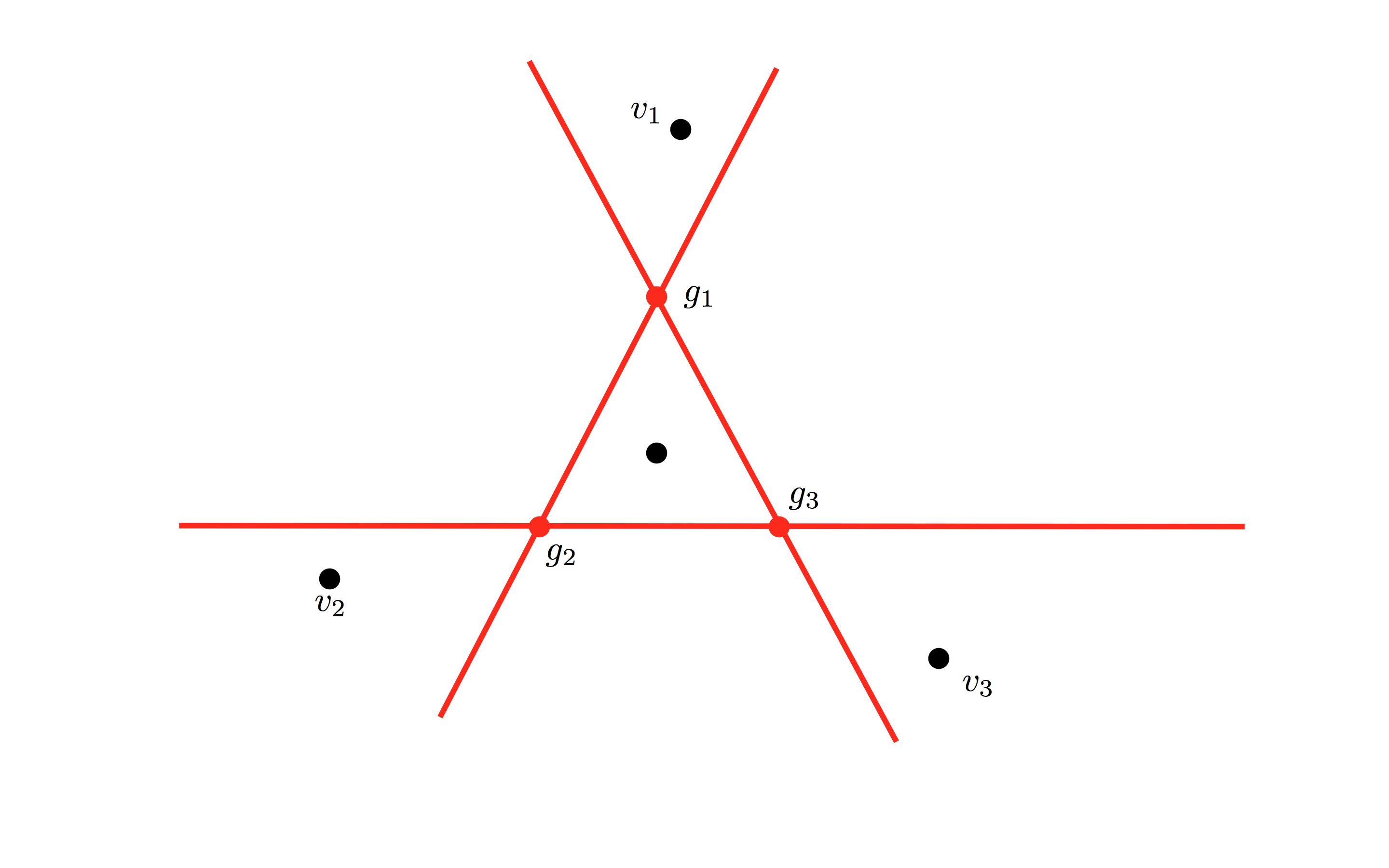}
\end{center}
\end{figure}
\end{center}

\begin{proof}[Proof of Lemma \ref{lem:approximation}]
~
\\
{Property 1:}

Let $u_i = \frac{1}{d^2} v_i$. For every $i \in [n]$, since $\ball_{4d^8}(0) \subseteq  \conv \{v_1, ..., v_{d + 1}  \} $, it holds that 
$$\ball_{d}\left(u_i \right) \subseteq \conv \{v_1, ..., v_{d + 1}  \} $$

Therefore, we can find integer points around $u_i$ in $\conv \{v_1, ..., v_{d + 1}  \} $. Now, let $g_i$ be the closest integer point to $u_i$, which has distance at most $d$ to $u_i$, i.e. $\| g_i -u_i \|_2 \le d$. Observe that $$\ball_{d^6}\left(0\right) \subseteq \conv \{2u_1, ..., 2u_{d + 1}  \} $$

Which implies that for every $i \in [d + 1]$, $\ball_{d^6/2}\left(u_i \right) \subseteq \conv \{2u_1, ..., 2u_{d + 1}  \} $. Therefore, $g_i \in \conv \{2u_1, ..., 2u_{d + 1}  \}  = \frac{2}{d^2} \set{P}_o$

Now we want to show that $0 \in \conv \{ g_1, ..., g_{d + 1}\}$. 

Consider a function $f: \conv\{ u_1, u_2, ..., u_{d + 1}\} \to \reals^d$ defined as: for $x = \sum_i \lambda_i' u_i$ where $(\lambda_1', \lambda_2', ..., \lambda_{d + 1}') \in \Delta^{d + 1}$: 
$$f(x) =  \sum_i \lambda_i' g_i$$

Observe that
$$\|f(x) - x \|_2 = \left\| \sum_{i = 1}^{d + 1} \lambda_i'  \left( u_i - g_i \right)\right\|_2 \le d$$

Notice that for $x_1, x_2 \in \conv\{ u_1, u_2, ..., u_{d + 1}\} $, $$f\left(\frac{x_1 + x_2}{2} \right) = \frac{f(x_1) + f(x_2)}{2}$$

Which implies that $f$ is a linear transformation. Moreover, $\ball_{d^2}(0)  \subseteq  \conv\{ u_1,u_2, ..., u_{d + 1}\} $. Therefore, $f(\ball_{d^2}(0)) = \cup_{x \in \ball_{d^2}(0)} \{ f(x) \}$ is a convex set, since a linear transformation preserves convexity.

Now, we want to show that $0 \in f(\ball_{d^2}(0) )$. Suppose on the contrary $0 \notin f(\ball_{d^2}(0) )$, then we know there is a separating hyperplane going through $0$ that separates  $0$ and $f(\ball_{d^2}(0) )$. Which implies that there is a point $g' \in \partial \ball_{d^2}(0) $ such that $$\textsf{dist}(g', f(\ball_{d^2}(0) )) = \min_{x \in f(\ball_{d^2}(0)) } \| x - g'\| \ge d^2$$

In particular, since $f(g') \in f(\ball_{d^2}(0) )$, the above equality implies $\textsf{dist}(g', f(g') )\ge d^2 $, in contradiction to $\|f(x) - x \|_2  \le d$. Therefore, $0 \in f(\ball_{d^2}(0) ) \subseteq \conv \{ g_1, ..., g_d \}$. 

We proceed to argue about the coefficients. Denote $g_i = u_i + b_i$, and by the above  $\|b_i \|_2 \le d$.   By symmetry  it  suffices  to show that $\frac{1}{2} \lambda_1' \le \lambda_1 \le 2 \lambda_1'$. 

Let $\{\lambda'_i\} \in \Delta^{d + 1}$ be such that   $\sum_i \lambda_i' g_i = \sum_i \lambda_i' (u_i + b_i) = 0$. Then 
$$ \sum_i \lambda_i' u_i =  - \sum_i \lambda_i' b_i$$

Since $\|b_i \|_2 \le d$, by the triangle inequality it holds that $\| \sum_i \lambda_i' b_i \|_2 \le d$, which implies
$$\ \left\|\sum_i  \lambda_i' u_i \right\| \le d$$

Let $H$ be the hyperplane going through $u_2, ..., u_{d + 1}$. Without lost of generality, we can apply a proper rotation (unitary transformation) to put $H = \{ x_1 = -a \}$ for some value $a > 0$, where $x_1$ denotes the first axis. 

Now, (after rotation) Define $b = (b_1, ..., b_d) = \sum_i \lambda_i' u_i $ and denote $u_1 = (a_1, ..., a_d)$. The point $b$ is a convex combination of $u_1$ and $c := \frac{1}{1 - \lambda_1'}\sum_{j \ge 2} \lambda_j' u_j$. In addition we know that $c_1 = -a$. Thus, we can write $\lambda_1'$ as:
$$\lambda_1'  =  \frac{b_1 - c_1}{(u_1)_1 - c_1 } = \frac{b_1 + a}{a_1 + a}$$

On the other hand, by $\sum_i \lambda_i u_i = 0$, we know that 
$$\lambda_1 = \frac{a}{a_1 + a}$$

Note that $\| b\|_2 \le d$, which implies $|b_1| < d$. However, by assumption there is a ball centered at $0$ of radius $4d^6$ in $\conv \{u_1, ..., u_{d + 1}  \} $, which implies $a \ge 4d^6 \ge 4|b_1|$. 

Therefore $\frac{1}{2} \lambda_1' \le \lambda_1 \le 2 \lambda_1'$.

\medskip
{Property 2:}

By symmetry, it suffices to show for $v_1$. 
there exists $\lambda_1^1 \ge \frac{1}{2d^2}$ and $\lambda_j^1 \ge 0 (j = 2, 3, ... , d+ 1)$, $\lambda_1^1 + \sum_{j = 2}^{d + 1} \lambda_j^1 = 1 $ such that
$$g_1 = \lambda_1^1 v_1 + \sum_{j = 2}^{d + 1} \lambda_j^1 g_j $$

Consider a function $f: \conv\{ v_1, u_2, ..., u_{d + 1}\} \to \reals^d$ defined as: for $x = \lambda' v_1 + \sum_{j = 2}^{d + 1} \lambda_j'  u_j$ where $(\lambda', \lambda_2', ..., \lambda_{d + 1}') \in \Delta^{d + 1}$: 
$$f(x) = \lambda' v_i + \sum_{j = 2}^{d + 1}\lambda_j' g_j$$

Observe that
$$\|f(x) - x \|_2 = \left\| \sum_{j = 2}^{d + 1} \lambda_j'  \left( u_j - g_j \right)\right\|_2 \le d$$

Notice that for $x_1, x_2 \in \conv\{ v_1, u_2, ..., u_{d + 1}\} $, $$f\left(\frac{x_1 + x_2}{2} \right) = \frac{f(x_1) + f(x_2)}{2}$$

Which implies that $f$ is a linear transformation. Moreover, $\ball_{d^2}(g_1) \subseteq \ball_{2d^2}\left(u_1\right) \subseteq  \conv\{ v_1,u_2, ..., u_{d + 1}\} $. Therefore, $f(\ball_{d^2}(g_1)) = \cup_{x \in \ball_{d^2}(g_1)} \{ f(x) \}$ is a convex set, since a linear transformation preserves convexity.

Now, we want to show that $g_1 \in f(\ball_{d^2}(g_1) )$. Suppose on the contrary $g_1 \notin f(\ball_{d^2}(g_1) )$, then we know there is a separating hyperplane going through $g_1$ that separates  $g_1$ and $f(\ball_{d^2}(g_1) )$. Which implies that there is a point $g' \in \ball_{d^2}(g_1) $ such that $$\textsf{dist}(g', f(\ball_{d^2}(g_1) )) = \min_{x \in f(\ball_{d^2}(g_1)) } \| x - g'\| = d^2$$

In particular, since $f(g') \in f(\ball_{d^2}(g_1) )$, the above equality implies $\textsf{dist}(g', f(g') )\ge d^2 $, in contradiction to $\|f(x) - x \|_2  \le d$ for all $x \in \conv\{ v_1, u_2, ..., u_{d + 1}\} $.

Therefore, there is a point $g \in \ball_{d^2}(g_1)$ such that $f(g) = g_1$, i.e. $g_1$ can be written as 
$$ g_1  = \lambda' v_1 + \sum_{j = 2}^{d + 1}\lambda_j' g_j \quad (\lambda' , \lambda_2', ..., \lambda_{d + 1}' ) \in \Delta^{d + 1}$$

We proceed to give a bound on the coefficients. Since $g_1 = f(g)$, we know that 
$$ g = \lambda' v_1 + \sum_{j = 2}^{d + 1}\lambda_j' u_j $$

On the other hand, observe that (since $\sum_j \lambda_j u_j = 0$ as defined in Property 1)$$u_1 = \frac{1}{d^2} v_1 + \left( 1 - \frac{1}{d^2} \right) \sum_{j = 1}^{d + 1}\lambda_ju_j $$

By $\| g - u_1 \|_2 \le 2d^2$, using the same method as Property 1 we can obtain: $\lambda' \ge \frac{1}{2 d^2}$

Which completes the proof.

\end{proof}

\noindent Now we can prove the discretization Lemma.  The proof goes by the following steps: 
\begin{enumerate}
\item
First, suppose the Lemma does not hold, then we can find a large hypercube that is contained inside $\K'$ and has entirely small \LCE\ compared to the value of the point $y$. 
\item
We proceed to identify the points whose $\tf_{min}$ value is associated with the LCE of the large hypercube, these $\tf_{min}$  have small values (compare to $F(y)$) and span a large region.

\item
We find a simplex of $d+1$ points that span a large region in which the same holds, i.e. $\tf_{\min}$ value compared to $v(y)$. 
\item
Using the approximation Lemma, we find an inner simplex of  $d+1$ {\it discrete} points inside the  previous simplex. These discrete points all have $\tf_{\min}$ value larger than $f(y)$ by the fact that they are inside the first large region.  
\item
We use the definition of $\tf_{\min}$ to show that one of the vertices of the outer simplex has value of $\tf_{\min}$ larger than $f(y)$, in contradiction to the previous observations. 
\end{enumerate}

\begin{proof}[Proof of Lemma \ref{lem:geometry}]
~
\\
{Step 1:}

Consider a point $y \in P$ with $\ball_{r} (y ) \in \set{K}$. By convexity of $F$, there is a hyperplane $H$ going $y$ such that on one side of the hyperplane, all the points have larger or equal $F$ value than $F(y)$. Therefore, there exists a point $y'$,  a cube $\cube_{r'}(y') \subset \ball_{r} (y ) $ centered at $y'$ with radius $r' = \frac{r}{2 \sqrt{d}}$ such that for all integer points $z \in \cube_{r'}(y') $, $F(z) \ge F(y)$. Let $v_1, ..., v_{2^d}$ be the vertex of this cube.

If there exists $i \in [2^d]$ such that $\FLCE(v_i) \ge \frac{1}{2}F(y)$, then we already conclude the proof. Therefore, we can assume that for all $i \in [2^d]$, $\FLCE(v_i) < \frac{1}{2} F(y)$. 
\medskip
{Step 2:}

By the definition of $\FLCE$, we know that for every $i \in  [2^d]$, there exists $p_{i, 1}, ...., p_{i, m} \in \set{K}'$ such that $v_i = \sum_{j} \lambda_{i, j} p_{i, j}, (\lambda_{i, 1}, ..., \lambda_{i, m} ) \in \Delta^m$ with 
$$\FLCE(v_i) = \sum \lambda_{i, j} \tF_{\min}(p_{i, j}) < \frac{1}{2}F(y)$$

Moreover, by Carathéodory's theorem \footnote{The original Carathedory's theorem only states for convex combination of points, but the same proof can be extended to convex functions by looking at the graph of the function}, we can make $m = d + 1$. 

Now we get a set of $(d + 1) 2^d$ many points $P_o = \{p_{i, j } \}_{i \in [2^d], j \in [d + 1]}$. Consider a size $d + 1$ subset $J =\{ p_{i_1, j_1}, ..., p_{i_{d + 1}, j_{d + 1}} \}$ of $P_o$, define convex set 
$$S_J = \left\{ x \in \cube_{r'}(y')  \mid \exists (\lambda_1, ..., \lambda_{d + 1}) \in \Delta^{d + 1}:  x = \sum_{s} \lambda_s  p_{i_s, j_s},   \sum_s \lambda_s \tF_{\min}(p_{i_s, j_s}) < \frac{1}{2}F(y) \right\}$$

We claim that $$\mathop{\bigcup}_{J \subseteq P_0, |J| = d+1} S_J = \cube_{r'}(y')$$

This is because for every $x \in \cube_{r'}(y') $, there exists $v_{i_1}, ..., v_{i_{d + 1}}$ and $\lambda_1, ..., \lambda_{d + 1} \in \Delta^{d + 1}$ such that 
$$\sum_{s \in [d + 1]} \lambda_s v_{i_s} = x$$

Moreover, for each $v_{i}$, $v_i = \sum_{j} \lambda_{i, j} p_{i, j}$. Therefore:
$$x = \sum_{s \in [d + 1]} \lambda_s \left(\sum_{ j} \lambda_{i_s, j} p_{i_s, j} \right)$$

On the other hand, $\sum_{s \in [d + 1]} \lambda_s \left(\sum_{ j} \lambda_{i_s, j} \tF_{\min}(p_{i_s, j} )\right) < \frac{1}{2} F(y)$. By Carathéodory's theorem, we can make the sum only contains $d + 1$ such $p_{i_s, j}$, which proves the claim. 

\medskip
{Step 3:}

By lemma \ref{lem:convex_cover}, we know that there exists $J^*$ such that $S_{J^*}$ contains a ball $\ball_{r''}(y'')$ inside $Q_{r'} (y')$ of radius $r'' = \frac{r'}{2 \sqrt{d} k d^d}$ where $k = { (d + 1) 2^d \choose d+ 1} \le 2^{2d^2}$ and $y''$ is an integer point.

For simplicity, we denote $J^* = \{ p_{1}, ..., p_{d+1} \}$. By the definition of $S_{J^*}$, there exists $(\lambda_1'', ..., \lambda_{d + 1}'') \in \Delta^{d + 1}$ such that
\begin{enumerate}
\item $$\sum_i \lambda_i'' p_i = y''$$
\item $$ \sum_i \lambda_i'' \tF_{\min} (p_{i}) < \frac{1}{2} F(y)$$
\end{enumerate}

 \medskip
 {Step 4:}Let $P = \conv \{ p_1, ..., p_{d + 1} \}$ with center $y''$. The above argument implies that $\ball_{r''}(y'') \subseteq \conv \{ p_1, ..., p_{d + 1} \}$, when $ r'' = \frac{r'}{2 \sqrt{d} k d^d} \ge \frac{r}{(2d)^{2d^2}} \ge 4 d^8$. By lemma \ref{lem:approximation},  there exists integer points $g_1, ..., g_{d + 1} \in \frac{2}{d^2} P$ (where $\frac{2}{d^2} P$ denotes shrink $P$ of factor $\frac{2}{d^2}$ according to center  $y''$) with
\begin{enumerate} 
\item
$y'' \in \conv\{g_1, ..., g_{d + 1}\}$
\item
For every $i \in [d + 1]$,  there exists $\{ \lambda^i_j \in \Delta^{d+1} \}$ such that $\lambda_i^i  \ge \frac{1}{2d^2}$ and
$$g_i = \lambda^i_i p_i + \sum_{j \not= i} \lambda_j^i g_j $$
\end{enumerate}

\noindent The conditions of the lemma assert that   $\frac{2}{d^2}\convexSet' \subseteq \frac{1}{2}\convexSet$, by $y'' \in \frac{1}{2 } \convexSet, P \subseteq \K'$, we know that $\frac{2}{d^2} P \subseteq \convexSet$.
This  implies that $g_i \in \set{Z}^d \cap \convexSet$, over which the RDF is defined, and we have values $\vLRA(g_i)$ and $\varLRA(g_i)$ to construct $\tF_{\min}$.

\paragraph{Step 5:} By the fact that $g_i = \lambda_i p_i^i + \sum_{j \not= i} \lambda_{j}^i g_j $ and the definition of $\tF_{\min}$, we know that 
$$\tF_{\min}(p_i) \ge \frac{1}{\lambda_i^i } \left( \vLRA(g_i) - \varLRA(g_i) - \sum_{j \not= i} \lambda_{ j}^i [\vLRA(g_j) + \varLRA(g_j) ]\right) $$

Let us write $y''  = \sum_i \lambda_i' g_{i}: (\lambda_1', ..., \lambda_{d+1}' ) \in \Delta^{d + 1}$. By the fact that $p_i = \frac{1}{\lambda_i^i} (g_i - \sum_{j \not= i} \lambda_{j}^ig_j)$ We can calculate:
$$y'' = \sum_i \lambda_i'' p_i, \quad \text{where } \frac{\lambda_i''}{\lambda_i^i}  - \sum_{j \not= i} \frac{\lambda''_j \lambda_{i}^j}{\lambda_j^j} = \lambda'_i $$

From Lemma \ref{lem:approximation} we also obtain that $\lambda_i'' \le 2 \lambda_i'$.

Moreover, for the interpolation:
\begin{eqnarray*}
 && \sum_i \lambda_i'' \tF_{\min} (p_{i}) \\
 &\ge&  \sum_i  \frac{\lambda_i''}{\lambda_i^i } \left( \vLRA(g_i) - \varLRA(g_i) - \sum_{j \not= i} \lambda_{j}^i [\vLRA(g_j) + \varLRA(g_j) ]\right) 
 \\
 &=&  \sum_i  \frac{\lambda_i''}{\lambda_i^i } \left( \vLRA(g_i) + \varLRA(g_i) - \sum_{j \not= i} \lambda_{j}^i [\vLRA(g_j) + \varLRA(g_j) ]\right)  - 2 \sum_{i } \frac{\lambda_i''}{\lambda_i^i } \varLRA(g_i) 
 \\
 &=& \sum_{i } \left(\left[\frac{\lambda_i''}{\lambda_i^i} - \sum_{j \not= i} \frac{\lambda_j'' \lambda_{i}^j}{\lambda_j^j}\right] [ \vLRA(g_i) + \varLRA(g_i)]\right) -  2 \sum_{i } \frac{\lambda_i''}{\lambda_i^i } \varLRA(g_i)
 \\
 &= & \sum_{i }\lambda_i' [\vLRA(g_i)  + \varLRA(g_i)] - 2 \sum_{i } \frac{\lambda_i''}{\lambda_i^i } \varLRA(g_i)
 \\
  &\ge&  \sum_{i }\lambda_i' [\vLRA(g_i)  + \varLRA(g_i)] - 4d^2  \sum_{i } \lambda_i'\varLRA(g_i)  \quad \quad  \mbox{ since $\lambda_i^i \geq \frac{1}{d^2}$,$\lambda_i'' \le 2 \lambda_i'$ }
  \\
  &=& \sum_{i }\lambda_i' [\vLRA(g_i)  - (4d^2  - 1) \varLRA(g_i)] 
\end{eqnarray*}

By assumption, since $g_i$ is a integer point, we get $\vLRA(g_i)  - (8 d^2 + 1) \sigma(g_i)  \ge 0$
\begin{eqnarray*}
\vLRA(g_i)  - (4d^2  - 1) \varLRA(g_i) &\ge& \vLRA(g_i) + \sigma(g_i) - 4d^2 \varLRA(g_i) 
\\
&\geq& F(g_i) - 4 d^2 \sigma(g_i) \quad \quad  \text{   since by definition } v(g_i) + \sigma(g_i) \ge F(g_i)
\\
&\geq&  F(g_i) - \frac{v(g_i) - \sigma(g_i)}{2} \quad \quad \text{ since } (8 d^2 + 1) \sigma(g_i)   \le v(g_i)
\\
&\ge& \frac{1}{2} F(g_i) \quad \quad  \text{   since by definition } v(g_i) - \sigma(g_i) \le F(g_i)
\end{eqnarray*}

Note that by the convexity of $F$, $\sum_{i } \lambda_i' F(g_i) \ge F(y'') \ge F(y)$ (last inequality is due to our choice of $y''$). Thus, 
$$\sum_i \lambda_i'' \tF_{\min} (p_{i}) \ge \sum_i \frac{1}{2} \lambda_i' F(g_i) \geq  \frac{1}{2}  F(y)$$

By contradiction we complete the proof.

\end{proof}

\section{Algorithm and statement of results} \label{sec:algorithm}

\subsection{Algorithm statement and parameter setting}

\begin{enumerate}
\item
$\delta > 0$ - an upper bound on the failure probability of the algorithm
\item
The desired regret bound $\ell = \lvalue$ 
\item
resolution of the grid: $\gridScale  = 2^{3d^2} \log^3 T \ge \extendRatio  \centerShrink^2 \sqrt{d}$.
\item
Scaling factor $\beta = \textendnumEpoch$.
\item
Extension ratio: $\extendRatio = \extendnumEpoch$.
\item
Blow up factor $\eta = 8d^2 + 1$.
\item 
Upper bound on the number of epoch $\epo \le \numEpoch$.
\end{enumerate}

\begin{algorithm}[h!]
\caption{ Bandit Ellipsoid \label{alg:1} } 
\begin{algorithmic}[1]
\STATE Input: A convex set $\set{K} \subseteq \reals^d$,  $\LRA$: a high-probability low regret bandit  algorithm on discrete set of points
\STATE Initialize: Epoch $\epo = 0$, epoch set $\Epo_{\epo} = \emptyset$, $\convexSet_{\epo} = \set{K}$, Grid $\grid_{\epo} =  \grid(\centerShrink \convexSet_{\epo} \cap \convexSet)$
\FOR{$t = 1$ to $T$}
\STATE  Apply the low-regret algorithm on current grid:
$$(p_t, \vLRA_t, \varLRA_t ) \leftarrow \LRA(\grid_{\epo}, \{p_i, x_i, f_i(x_i) \mid i \in \Epo_{\epo} \} )$$
where $p_t,\vLRA_t,\varLRA_t$ are defined as in section \ref{sec:lradefinitions}. 

\STATE Play  a point $x_t \in \grid_{\epo}$ from distribution $p_t (x_t)$,  observe value $f_t(x_t)$. Set $ \Epo_{\epo} = \Epo_{\epo} \cup \{t\}$. 
\STATE (Shift): Shift $v_{t}$ by a constant so that $\min_{x \in \grid_{\epo}} \{v_{t}(x) - \blowUp \sigma_{t} (x) \}= 0$, for simplicity we just keep the same notation for the new $v_{t}$. Moreover, we can shift $F^{\epo} = \sum_{j \in \Epo_{\epo}} f_j$ by the same constant and assume that adversary presents us the (shifted) $f_j$. For simplicity we also keep the same notation for the new $F^{\epo}$.  
\STATE Compute $\FLCE^\tau = \text{FitLCE}(\centerShrink \K_\tau \cap \convexSet,[ \grid_{\epo}, v_\epo,\sigma_\epo])$. 
\IF{ $\forall x \in  \convexSet_{\epo}, \exists j \le \epo,  \FLCE^{j} (x) > \frac{\levelSetValue}{4}$} 
\STATE RESTART (goto Initialize)
\ENDIF

\IF{(\DM) \ $\exists \tilde{x}_{\tau} \in \frac{\convexSet_{\epo}}{\centerShrink}$ such that $\FLCE^\tau (\tilde{x}_{\tau}) \ge \levelSetValue$ }

\STATE $\convexSet_{\epo + 1}  = \text{ShrinkSet}(\convexSet_{\epo} , \tilde{x}_{\tau}, \FLCE^\tau, \grid_{\epo}, \{  \vLRA_t, \varLRA_t  \})$ 
\STATE Set  $ \Epo_{\epo + 1} = \emptyset$, $\grid_{\epo + 1} = \text{grid} (\centerShrink \convexSet_{\epo + 1} \cap \convexSet, \alpha)$, $\epo = \epo + 1$.
\ENDIF
\ENDFOR
\end{algorithmic}
\label{alg:main}
\end{algorithm}

\begin{center}
\begin{figure}
\begin{center}
\caption{Depiction of the algorithm}
\includegraphics[width=0.6\textwidth]{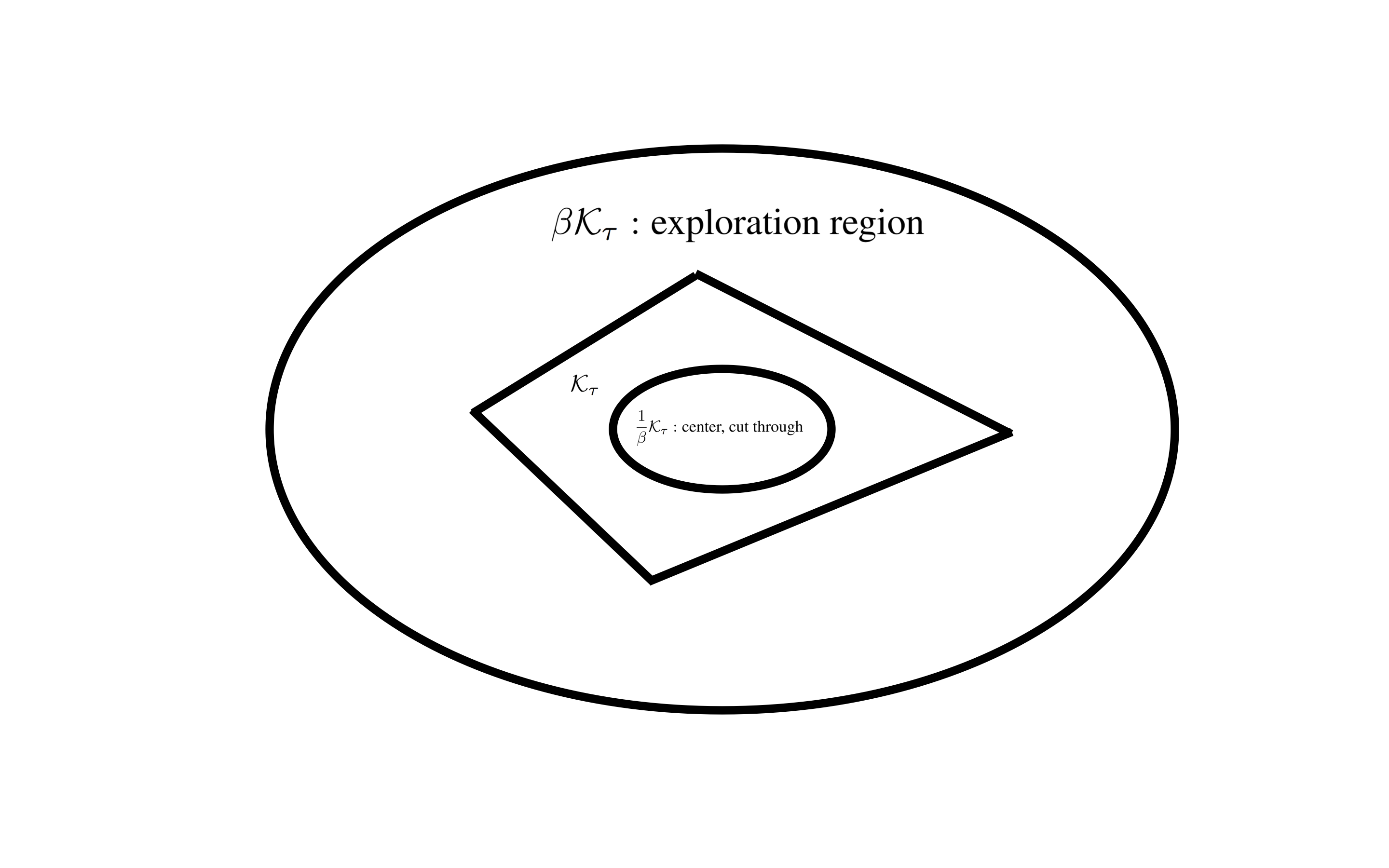}
\end{center}
\end{figure}
\end{center}

This algorithm calls upon two subroutines, FitLCE which was defined in section \ref{sec:geometry}, and ShrinkSet which we now define.

\begin{algorithm}[h!]
\caption{ ShrinkSet } 
\begin{algorithmic}[1]
\STATE Input: Convex set $\convexSet_{\epo} $, convex function $\FLCE^\tau$, point $\tilde{x}_{\tau} \in \K_\tau$, Grid $\grid_{\epo}$, value estimation $\vLRA_t$ and variance estimation $\varLRA_t$. 
\STATE Compute a separation hyperplane $H_{\epo}'$ through $\tilde{x}_{\tau}$ between $\tilde{x}_{\tau}$ and $ \{ y \mid \FLCE^{\tau}(y) < \levelSetValue \}$. Assume $H_{\epo}' = \{ x \mid \langle h_{\epo}, x  \rangle   = w_{\epo}  \}$ and $ \{ y \mid \FLCE^{\tau}(y) < \levelSetValue \} \subseteq \{ y \mid \langle h_{\epo}, x  \rangle   \le w_{\epo}  \}$
\STATE Let $x_{\epo}$ be the center of the \MVEE \ $\ellipsoid_{\epo}$ of $\convexSet_{\epo} $. 
\STATE (Amplify Distance). Let $H_{\epo} = \{ x \mid \langle h_{\epo}, x  \rangle   = z_{\epo} \} $ for some $ z_{\epo} \ge 0$ such that the following holds: 
\begin{enumerate}
\item $ \{ y \mid \FLCE^{\tau}(y) < \levelSetValue \} \subseteq \{ y \mid \langle h_{\epo}, y  \rangle   \le z_{\epo}  \}$
\item $\distance(x_{\epo}, H_{\epo}) = 2 \distance(x_{\epo}, H_{\epo}')$. 
\item $\langle h_{\epo}, x_{\epo}  \rangle   \le z_{\epo}$. 
\end{enumerate}
\STATE Return $\convexSet_{\epo + 1} =  \left( \convexSet_{\epo} \cap \{y  \mid \langle h_{\epo}, y  \rangle   \le z_{\epo} \} \right)$
\end{algorithmic}
\label{alg:shrink_set}
\end{algorithm}

\begin{center}
\begin{figure}[h!]
\begin{center}
\caption{Depiction of the ShrinkSet procedure}
\includegraphics[width=0.6\textwidth]{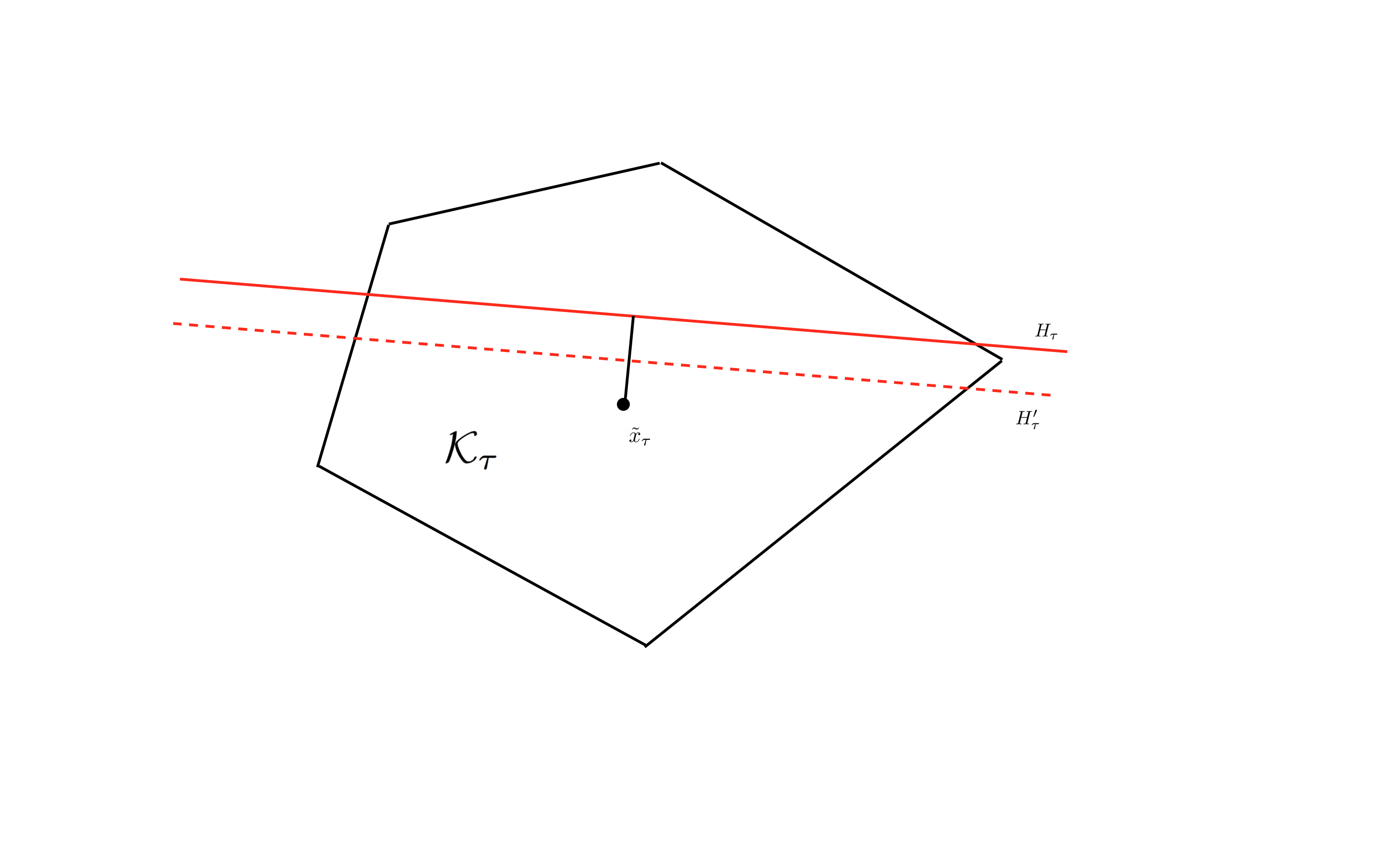}
\end{center}
\end{figure}
\end{center}

\subsection{Statement of main theorem}

\begin{thm}[Main, full algorithm]\label{thm:main} Suppose for all time $t$ in all epoch $\epo$, $\LRA$ outputs $v_t$ and $\sigma_t$ such that for all $x \in \grid_{\epo}$, $\left(\sum_{j \in \Epo_{\epo}, j \le t} f_j(x)\right) \in [v_t(x) - \sigma_t(x), v_t(x) + \sigma_t(x) ]$. Moreover, $\LRA$ achieves a value 
$$v_{\tau} (\set{A}) = \sum_{j \in \Epo_{\epo}, j \le t} f_j(x_j) \le \min_{x \in \grid_{\epo}} \left\{ v_t(x) - \blowUp \sigma_t(x) \right\}  + \frac{\levelSetValue}{\valObtain}$$

 then Algorithm \ref{alg:1}  satisfies 
$$ \sum_t f_t(x_t) - \min_{x^*} \sum_t f_t(x^*) \leq  \levelSetValue$$  
\end{thm}

\begin{cor}[Exp3.P.]\label{cor:main} Algorithm \ref{alg:main} with $\set{A}$ being Exp3.P satisfies the condition in Theorem \ref{thm:main} with probability $1 - \delta$ for $$\levelSetValue = \lvalue$$
\end{cor}

\subsection{Running time}

Our algorithm runs in time $O\left( (\log T)^{\poly(d)} \right)$, which follows from Theorem \ref{thm:LCE} and the running time of Exp3.P on $K \le (2d \gridScale)^{d} $ Experts

\subsection{Analysis sketch}

Before going to the details, we briefly discuss the steps of the proof.

{\bf Step 1:} In the algorithm, we shift the input function so that the player can achieve a value $\le \sqrt{T}$. Therefore, to get the regret bound, we can just focus on the minimal value of $\sum_t f_t$.

{\bf Step 2:} We follow the standard Ellipsoid argument, maintaining a shrinking set, which at epoch $\tau $ is denoted $\K_\tau$. We show the volume of this set decreases by a factor of at least $1 - \frac{1}{d}$, and hence the number of epochs between iterative RESTART operations can be bounded by $O( d^2 \log T)$ (when the diameter of $\K_\tau$ along one direction decreases below $\frac{1}{\sqrt{T}}$, we do not need to further discretizate along that direction).  

{\bf Step 3:} We will show that inside each epoch $\epo$, for every $x \in \K_{\epo}$, $\sum_{t: t \text{ in epoch }\tau} f_t(x)$ is lower bounded by $- \frac{2\ell}{\gamma} $ for $\ell \approx \sqrt{T}, \gamma \ge 1$. For point $x$ outside the $\K_{\tau}$, $\sum_{t: t \text{ in epoch }\tau} f_t (x)$ is lower bounded by $- \frac{2\ell}{\gamma} \gamma(x,  \K_{\epo})$.

{\bf Step 4:} We will show that when one epoch $\epo$ ends, for every point $x$ cut off by the separation hyperplane, $\sum_{t: t \text{ in epoch }\tau} f_t (x)$ is lower bounded by $ \frac{ \ell}{2} \gamma(x,  \K_{\epo})$

{\bf Step 5:} Putting the result of 3, 4 together, we know that for a point outside the current set $\K_{\epo}$, it must be cut off by a separation hyperplane at some epoch $j \le \epo$. Moreover, we can find such $j$ with $ \gamma(x,  \K_{j}) \ge \frac{ \gamma(x,  \K_{\tau}) }{d}$. Which implies that 
$$\sum_t f_t(x) = \sum_{t: t \text{ in epoch }1, 2, ..., j - 1,  j + 1, ..., \tau} f_t(x) +  \sum_{t: t \text{ in epoch } j} f_t(x) \ge  - \frac{2 \tau  \ell}{\gamma} \gamma(x,  \K_{\epo}) + \frac{\ell \gamma(x,  \K_{\epo}) }{2d} \approx \sqrt{T}$$

By our choice of $\gamma \ge 8d \tau$. Therefore, when the adversary wants to move the optimal outside the current set $\K_{\epo}$, the player has zero regret. Moreover, by the result of 3, inside current set $\K_{\epo}$, the regret is bounded by $\tau \frac{2\ell}{\gamma} \approx \sqrt{T}$. 
\\
The crucial steps in our proof are {\bf Step 3} and {\bf Step 4}. Here we briefly discuss about the intuition to prove the two steps.

{\bf Intuition of Step 3:} For $x \in \K_{\tau}$, we use the grid property (Property of grid, \ref{property:grid}) to find a grid $x_g$ point such that $x_c = x_g + \gamma (x_g - x)$ is close to the center of $\K_{\tau}$. Since $x_g$ is a grid point, by shifting we know that $$\sum_{t: t \text{ in epoch }\tau} f_t(x_g) \ge 0$$

Therefore, if $\sum_{t: t \text{ in epoch }\tau} f_t(x)< - \frac{2\ell}{\gamma}$, by convexity of $f_t$, we know that $\sum_{t: t \text{ in epoch }\tau} f_t(x_c) \ge 2 \ell$. Now, apply discretization Lemma \ref{lem:geometry}, we know that there is a point $x'_c$ near $x_c$ such that $\FLCE^{\tau} (x_c') \ge \ell$, by our \DM \ condition, the epoch $\tau$ should end. Same argument can be applied to $x \notin \K_{\tau}$.

{\bf Intuition of Step 4:} We use the fact that the algorithm does not RESTART, therefore, according to our condition, there is a point $x_0 \in \K_{\tau}$ with $\FLCE^{\tau}(x_0) \le \frac{\ell}{4}$. Observe that the separation hyperplane of our algorithm separates $x_0$ with points whose $\FLCE^{\tau} \ge \ell$. Using the convexity of $\FLCE$, we can show that $\FLCE^{\tau}(x) \ge \frac{ \ell}{2} \gamma(x,  \K_{\epo})$. Apply the fact that $\FLCE^{\tau}$ is a lower bound of $\sum_{t: t \text{ in epoch }\tau} f_t$ we can conclude $\sum_{t: t \text{ in epoch }\tau} f_t(x) \ge \frac{ \ell}{2} \gamma(x,  \K_{\epo})$.

Notice that here we use the convexity of $\FLCE$, and also the fact that it is a lower bound on $F$ (standard convex regression is not a lower bound on $F$, see section \ref{sec:regression} for further  discussion on this issue).

Now we can present the proof for general dimension $d$

To prove the main theorem we need the following lemma, starting from the following corollary of Lemma \ref{lem:geometry}:

\begin{cor}\label{lem:geo_co}
~
\\
(1). For every epoch $\epo$, $\forall x \in \centerShrink \convexSet_{\epo} \cap \convexSet$, $$\FLCE^{\epo} (x)\le F^{\epo}(x) = \sum_{i \in \Epo_{\epo}} f_i(x)$$.
\\
(2). For every epoch $\epo$, let $\centerPoint_{\epo}$ be the center of the \MVEE \ of $\convexSet_{\epo}$, then $F^{\epo}(x) = \sum_{i \in \Epo} f_i(x) \le 2\levelSetValue$.
\end{cor}

\begin{proof}

(1) is just due to the definition of LCE. (2) is due to the Geometry Lemma on $F^{\epo}$: for every $x \in  \frac{\convexSet}{2\centerShrink}$, there exists $x' \in \left(x + \frac{\convexSet_{\epo}}{2\centerShrink} \right) \subseteq \frac{\convexSet_{\epo}}{\centerShrink}$ such that $\FLCE^{\epo} (x' ) \ge \frac{1}{2}F^{\epo}(x)$
\end{proof}

\begin{lem}[During an epoch]\label{lem:inside}  
During every epoch $\epo$ the following holds:
$$ F^\tau(x)  \ge \mycases {-   \frac{2 \levelSetValue}{\twoExtendRatio}}{ x \in \set{K} \cap \convexSet_{\epo} }
{- \frac{2 \distanceRatio(x, \convexSet_{\epo}) \levelSetValue}{\twoExtendRatio} }{ x \in \set{K} \cap \convexSet_{\epo}^c} $$


\end{lem}

\begin{lem}[Number of epoch]  \label{lem:number}There are at most $\numEpoch$ many epochs before \emph{RESTART}.
\end{lem}

\begin{proof}[Proof of \ref{lem:number}]

Let $\set{E}_{\epo}$ be the minimal volume enclosing Ellipsoid of $\convexSet_{\epo}$, we will show that $$\vol(\set{E}_{\epo + 1}) \le \left(1 - \frac{1}{8d}\right)\vol(\set{E}_{\epo})$$

First note that $\convexSet_{\epo + 1} = \convexSet_{\epo} \cap \set{H}$ for some half space $\set{H}$ corresponding to the separating  hyperplane  going through $\frac{1}{\centerShrink}\set{E}_{\epo} $, therefore, $\convexSet_{\epo + 1}  \subset \set{E}_{\epo} \cap \set{H}$. 

Let $\set{E}_{\epo + 1}'$ be the minimal volume enclosing Ellipsoid of $\set{E}_{\epo} \cap \set{H}$, we know that 
$$\vol(\set{E}_{\epo}) \le \vol(\set{E}_{\epo + 1}')$$

Without lose of generality, we can assume that $\set{E}_{\epo} $ is centered at 0. Let $A$ be a linear operator on $\Real^d$ such that $A(\set{E}_{\epo})$ is the unit ball $\ball_1(0)$, observe that 
$$\frac{\vol(\set{E}_{\epo + 1}') }{\vol(\set{E}_{\epo}) } =\frac{\vol(A\set{E}_{\epo + 1}') }{\vol(A\set{E}_{\epo}) }  $$

Since $A\set{E}_{\epo + 1}' $ is the MVEE of $A\set{E}_{\epo} \cap A\set{H}$, where $A \set{H}$ is the halfspace corresponding to the separating hyperplane going through $\ball_{\frac{1}{\centerShrink}}(0)$. Without lose of generality, we can assume that $\set{H} = \{ x \in \reals^d \mid x_1 \ge a \}$ for some $a$ such that $|a| \le \frac{1}{\centerShrink} \le \frac{1}{d^2}$. 

Observe that $$A\set{E}_{\epo} \cap A\set{H} \subseteq \left\{ x \in \reals^d\left| \right. \frac{(x_1 - \frac{1}{4d})^2}{\left(1  - \frac{1}{4d}\right)^2} + \frac{x_2^2}{1 +\frac{1}{12d^2}} + ... + \frac{x_d^2}{1 + \frac{1}{12d^2}} \le 1 \right\} = \set{E}$$

Therefore, 
$$\vol(A\set{E}_{\epo + 1}' ) \le \vol(\set{E}) \le  1 - \frac{1}{8d}.$$

Now, observe that the algorithm will not cut through one eigenvector of the \MVEE \ of $\set{K}_{\epo}$ if its length is smaller than $\frac{1}{\sqrt{T}}$, and the algorithm will stop when all its eigenvectors have length smaller than $\frac{1}{\sqrt{T}}$. Therefore, the algorithm will make at most
$$d \log_{1 - \frac{1}{8d}} \left(\frac{1}{\sqrt{T}}\right) = \numEpoch$$
many epochs.
\end{proof}

\begin{lem}[Beginning of each epoch] \label{lem:begin}For every $\epo \ge 0$:

$$ \sum_{i=0}^{\tau-1} F^i(x)  \ge \mycases {-  \epo \frac{2 \levelSetValue}{\twoExtendRatio} }   { x \in \set{K} \cap \convexSet_{\epo} }
{ \frac{\distanceRatio(x, \convexSet_{\epo}) \levelSetValue}{64d}}{ x \in \set{K} \cap \convexSet_{\epo}^c} $$

\ignore{
For every $x \in \set{K} \cap \convexSet_{\epo}$, 
$$\sum_{i \in  \Epo_{0} \cup ... \cup \Epo_{\epo - 1}} f_i (x)  \ge -  \epo \frac{ \levelSetValue}{\twoExtendRatio} $$

 For every $x \in \set{K} \cap \convexSet_{\epo}^c$ 
$$\sum_{i \in  \Epo_{0} \cup ... \cup \Epo_{\epo - 1}} f_i (x) \ge \frac{\distanceRatio(x, \convexSet_{\epo}) \levelSetValue}{32d}$$
}

\end{lem}

\begin{lem}[Restart] \label{lem:restart} (After shifting) If $\LRA$ obtains a value $v_j (\set{A})= \sum_{t \in \Epo_j} f_t(x_t) \le \frac{\levelSetValue}{\valObtain}$ for each epoch $j$, then when the algorithm \emph{RESTART}, $\regret = 0$. 
\end{lem}

\subsection{Proof of main theorem}
Now we can prove the regret bound assuming all the lemmas above, whose proof we defer to the next section.
\begin{proof}[Proof of Theorem \ref{thm:main}]
Using Lemma \ref{lem:restart}, we can only consider epochs between two {RESTART}. Now, for epoch $\epo$, we know that for $x \in \set{K} \cap \convexSet_{\epo}^c$, 
$$\sum_{i \in  \Epo_{0} \cup ... \cup \Epo_{\epo - 1}} f_i (x) \ge \frac{\distanceRatio(x, \convexSet_{\epo}) \levelSetValue}{64 d}$$
$$\sum_{i \in  \Epo_{\epo}} f_i (x) \ge  - \frac{2 \distanceRatio(x, \convexSet_{\epo}) \levelSetValue}{\twoExtendRatio}$$

Therefore, for $x \in \set{K} \cap \convexSet_{\epo}^c$
$$\sum_{i \in  \Epo_{0} \cup ... \cup \Epo_{\epo }} f_i (x) \ge {\distanceRatio(x, \convexSet_{\epo}) \levelSetValue} \left( \frac{1}{64d} - \frac{2}{\twoExtendRatio} \right) \ge 0$$

By our choice of $\extendRatio = \extendnumEpoch$. 

In the same manner, we know that for $x \in \set{K} \cap \convexSet_{\epo}$, 
$$\sum_{i \in  \Epo_{0} \cup ... \cup \Epo_{\epo }} f_i (x) \ge - \frac{2(\epo + 1)\levelSetValue }{\twoExtendRatio} \ge  -  \frac{\levelSetValue }{2}$$

By $\epo \le \numEpoch$.

Which implies that for every $x \in \set{K}$, $\sum_{i \in  \Epo_{0} \cup ... \cup \Epo_{\epo }} f_i (x) \ge - \frac{\levelSetValue }{2}$. 

Denote by $v_j(\LRA) = \sum_{i \in \Epo_{j}, i \le t} f_j(x_j) $ the overall loss incurred by the algorithm in epoch $j$ before time $t$. The low-regret algorithm $\LRA$ guarantees that in each epoch:
\begin{eqnarray*}
v_j(\LRA) &= & \sum_{i \in \Epo_{\epo}, i \le t} f_i(x_i)  \\
& \leq& \min_{x \in \grid_{\epo}} \left\{ v_t(x) - \blowUp \sigma_t(x) \right\}  + \frac{\ell}{\valObtain } \\
& \leq& \frac{\ell}{\valObtain} \quad \quad  \mbox{ by shifting } \min_{x \in \grid_{\epo}} \left\{ v_t(x) - \blowUp \sigma_t(x) \right\} =0
\end{eqnarray*}
Thus $\LRA$ obtains over all epochs a total value of at most 
$$\sum_{ 0 \le j \le \epo}  \frac{\ell}{\valObtain } =  (\tau + 1) \times \frac{\ell}{\valObtain } \leq \frac{\levelSetValue}{2}.$$ 

Therefore, 
$$\regret = \sum_{ 0 \le j \le \epo}  v_j(\set{A}) - \sum_{i \in  \Epo_{0} \cup ... \cup \Epo_{\epo }} f_i (x^*)  \le \levelSetValue$$
\end{proof}

\section{Analysis and proof of main lemmas}\label{sec:analysis}

\subsection{Proof of Lemma \ref{lem:inside}}
\begin{proof}[Proof of \ref{lem:inside}] 
\medskip
{\bf Part 1:}

Consider any  $x \in  \convexSet_{\epo}$. 
By Lemma $\ref{property:grid}$ part 1, we know that there exists $x_g \in \grid_{\epo}$ such that $x_{c} = x_g + \extendRatio (x - x_g) \in \frac{\convexSet_{\epo}}{2 \centerShrink}$. 
Any convex function $f$ satisfies for any two points $y,z$ that  $f(\gamma x + (1-\gamma)y) \leq \gamma f(x) + (1-\gamma) f(y)$. Applying this to the convex  function $F^{\epo}$ over the line on which the points $x,x_c,x_g$ reside and observe $\gamma = \frac{\|x_c - x_g\|_2}{\|x_g - x\|_2} $, we have 
$${F^{\epo}(x_c) - F^{\epo}(x_g)} \ge \frac{||x_c - x_g||_2}{||x_g - x||_2} ({F^{\epo}(x_g) - F^{\epo}(x)}) = \extendRatio ({F^{\epo}(x_g) - F^{\epo}(x)}) $$ 

Since $x_g \in \grid$ and we shifted all losses on the grid to be nonnegative, $F^{\epo}(x_g) \ge 0$. Thus, we can simplify the above to:
$$ F^{\epo}(x_c)  \ge -  \extendRatio  F^{\epo}(x)  $$ 

Since the epoch is ongoing, the conditions of $\DM$ are not yet satisfied, and hence   $\forall x' \in \frac{1}{\beta}\K_\tau , \ \FLCE^{\epo}(x') \le  \ell$. 
By  (2) of Lemma \ref{lem:geo_co} for all points $x''$ in $\frac{1}{2 \beta} \K_\tau$ it holds that $F^\tau(x'') \leq  2\ell$, in particular $F^\tau(x_c) \leq 2\ell$. The above simplifies to
$$   F^{\epo}(x)  \geq - \frac{1}{\gamma} F^\tau(x_c) \geq - \frac{2 \levelSetValue}{\gamma} $$


\medskip
{\bf Part 2:} 

 For $x \in  \convexSet_{\epo}^c \cap \set{K}$
By Lemma $\ref{property:grid}$ part 2, we know that there exists $x_g \in \grid_{\epo}$ such that $x_{c} = x_g + \frac{ \extendRatio}{\distanceRatio(x,  \convexSet_{\epo})} (x - x_g) \in \frac{\convexSet_{\epo}}{\centerShrink^2}$. 
Now, by the convexity of $F^{\epo}$, we know that $${F^{\epo}(x_c) - F^{\epo}(x_g)}{} \ge \frac{||x_c - x_g||_2}{||x_g - x||_2}  \left( F^{\epo}(x_g) - F^{\epo}(x) \right)= \frac{\extendRatio}{\distanceRatio(x,  \convexSet_{\epo})}  \left( F^{\epo}(x_g) - F^{\epo}(x) \right)$$

Since $x_g \in \grid$ and we shifted all losses on the grid to be nonnegative, $F^{\epo}(x_g) \ge 0$. Thus, we can simplify the above to:
$${F^{\epo}(x_c)} \ge - \frac{\extendRatio}{\distanceRatio(x,  \convexSet_{\epo})}  F^{\epo}(x)$$

Since the epoch is ongoing, the conditions of $\DM$ are not yet satisfied, and hence   $\forall x' \in \frac{1}{\beta}\K_\tau , \ \FLCE^{\epo}(x') \le  \ell$. 
By  (2) of Lemma \ref{lem:geo_co} for all points $x''$ in $\frac{1}{2 \beta} \K_\tau$ it holds that $F^\tau(x'') \leq  2\ell$, in particular $F^\tau(x_c) \leq 2\ell$. The above simplifies to
$$   F^{\epo}(x)  \geq - \frac{\distanceRatio(x,  \convexSet_{\epo})}{\extendRatio}  F^{\epo}(x_c) \geq - \frac{ 2\distanceRatio(x,  \convexSet_{\epo}) \levelSetValue}{\gamma} $$

\end{proof}

\subsection{Proof of Lemma \ref{lem:begin}}
\begin{proof}[Proof of Lemma \ref{lem:begin}]
\medskip
{\bf Part 1:}
For every $x \in \set{K } \cap \convexSet_{\epo}$, since $K_\tau \subseteq \K_{\tau-1} \subseteq ... \subseteq \K_0 = \K$, we have  $x \in  \convexSet_{j}$ for every $0 \le j \le \epo$. Therefore, by Lemma \ref{lem:inside} we get $   F^j(x)    \ge -   \frac{ 2\levelSetValue}{\twoExtendRatio} $. Summing over the epochs, 
$$ \sum_{i=0}^{\tau - 1} F^i(x)  \ge -  \epo \frac{ 2\levelSetValue}{\twoExtendRatio} $$

\medskip
{\bf Part 2:}
Figure \ref{fig:shalom} illustrates the proof. 
\begin{figure}
\begin{center}
\includegraphics[width=0.8\textwidth]{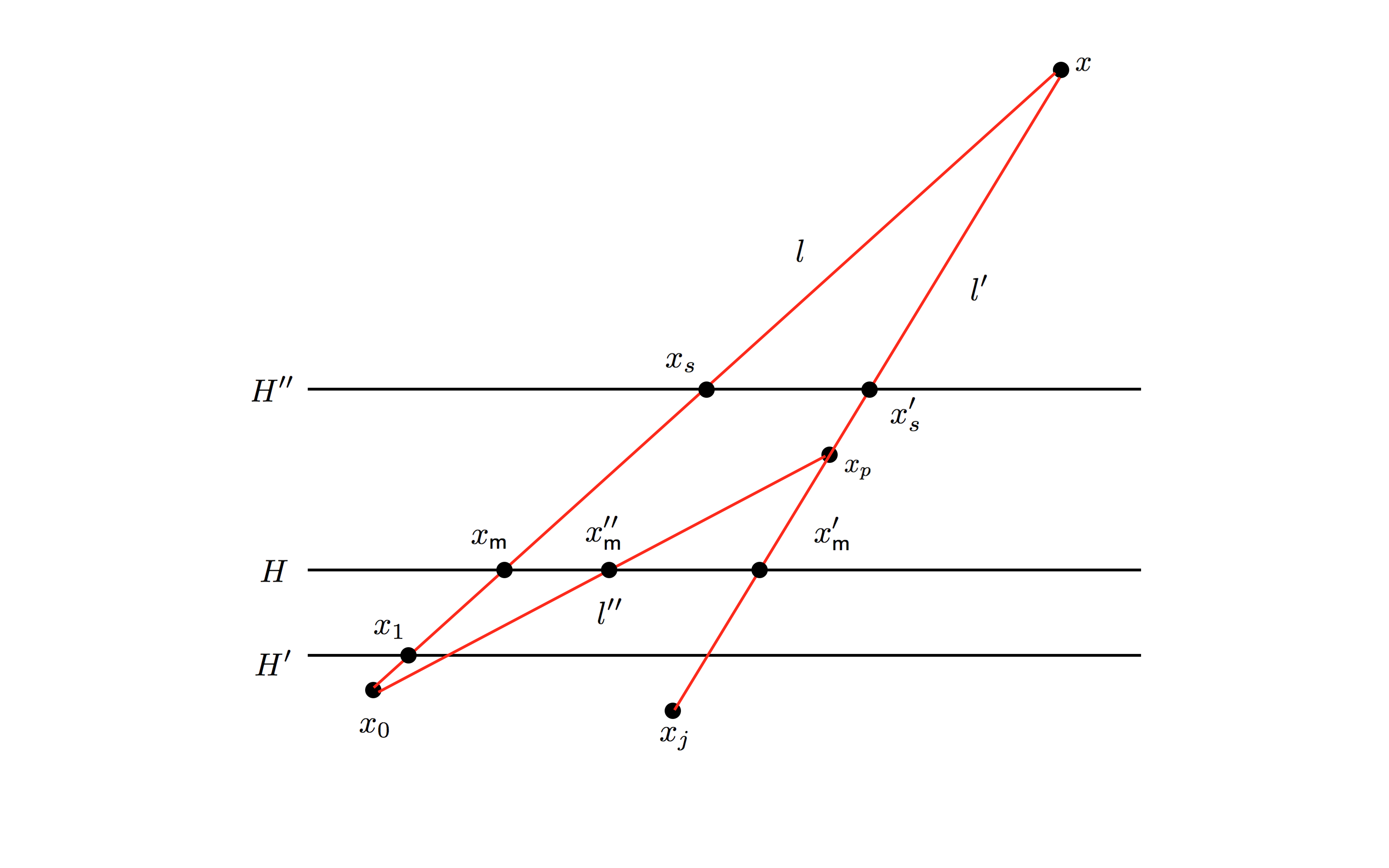} \caption{geometric intution for the proof \label{fig:shalom}}
\end{center}
\end{figure}

For every $x  \in \set{K} \cap \convexSet_{\epo}^c$, since the Algorithm does not RESTART, therefore, there must be a point $x_0 \in \convexSet_{\epo}$ such that 
\begin{equation} \label{eqn:shalom}
\forall \epo' \le \epo, \FLCE^{\epo'} (x_0) \le \frac{\levelSetValue}{4}
\end{equation}

Let $l$ be the line segment between $x$ and $x_{0}$. Since $x \notin  \convexSet_{\epo}$, the line $l$ intersects $\K_\tau$, and denote $\xint$ be the intersection point between $l$ and $\convexSet_{\epo}$: $\{ \xint\} = l \cap \convexSet_{\epo}$. 
The corresponding boundary of $\K_\tau$ was constructed in an epoch $j \le \epo$, and a hyperplane which separates the $\ell$-level set of $\K_j$, namely $H = \{ \xint \mid \langle h_j, \xint \rangle = z_j \})$ (See $\text{ShrinkSet}$ for definition of $h_j, z_j$) such that $H \cap l = \{\xint \}$.  

Now, by the definition of Minkowski Distance, we know that (Since Minkowski Distance is the distance ratio to $\frac{1}{d}\ellipsoid_{\epo}$ where $\ellipsoid_{\epo}$ is the \MVEE  \ of $\convexSet_{\epo}$,  $\frac{1}{d}\ellipsoid_{\epo}$ can be $1/d$ smaller than $\convexSet_{\epo}$, and $\xint$ is the intersection point to $\convexSet_{\epo}$)
 $$\frac{|| x - \xint||_2} {||\xint - x_{0}||_2}\ge   \frac{\distanceRatio(x, \convexSet_{\epo})  - 1}{2d}$$

We know that (by the convexity of $\FLCE^j$) 
$$\frac{\FLCE^j(x) - \FLCE^j(\xint) }{\FLCE^j(\xint)  - \FLCE^j(x_0) } \ge \frac{|| x - \xint||_2} {||\xint - x_{0}||_2}\ge   \frac{\distanceRatio(x, \convexSet_{\epo})  - 1}{2d}$$
where the denominator is non-negative, by equation \eqref{eqn:shalom},  $\FLCE^j(x_0) \le \frac{\levelSetValue}{4}$, and by the definition of $H$ (separation hyperplane of the $\levelSetValue$-level-set of $\FLCE^j$), $\FLCE^j(\xint) \ge \levelSetValue$. This implies 
$$\FLCE^j(x) \ge \frac{(\distanceRatio(x, \convexSet_{\epo}) -1) \cdot \frac{3}{4} \levelSetValue}{2d} + \ell  \ge \frac{\distanceRatio(x, \convexSet_{\epo})  \levelSetValue}{4d}$$

We consider the following two cases: (a). $x \in \centerShrink \convexSet_{j}$, (b). $x \notin \centerShrink \convexSet_{j}$. 

{\bf case (a): $x \in \centerShrink \convexSet_{j}$}

The LCE is a lower bound of the original function only for $x$ in the LCE fitting domain, here LCE = $\FLCE^j$, original function $F^j: \centerShrink \convexSet_j  \cap \convexSet \to \reals$, so it is only true for $x \in \centerShrink \convexSet_j \cap \convexSet$. 
Now, by (1) in Lemma \ref{lem:geo_co}, we know that $F^j(x)  \ge \FLCE^j(x) \ge \frac{\distanceRatio(x, \convexSet_{\epo}) \levelSetValue}{4d}$.

For other epoch $i < \tau$, we can apply Lemma \ref{lem:inside} and get $F^{i}(x) \ge - \frac{2\distanceRatio(x, \convexSet_{i}) \levelSetValue}{\twoExtendRatio}$. Since the set $ \K_{\tau}\subseteq \K_{\tau - 1}  \subseteq ... \subseteq \K_{0} $, By John's theorem, we can conclude that $\distanceRatio(x, \convexSet_{i}) \le 2d \distanceRatio(x, \convexSet_{\tau})$

which implies 
\begin{eqnarray*}
\sum_{i=0}^{\epo - 1} F^i (x)  &\ge&\sum_{i \not= j } F^i (x)  + F^j(x)
\\
&\ge& \frac{\distanceRatio(x, \convexSet_{\epo}) \levelSetValue}{4d}  - \epo  \times  \frac{4d \distanceRatio(x, \convexSet_{\epo}) \levelSetValue}{\twoExtendRatio}\ge \frac{\distanceRatio(x, \convexSet_{\epo}) \levelSetValue}{32d}
\end{eqnarray*}
by our choice of parameters $\epo \le  \numEpoch$ and $ \extendRatio = \extendnumEpoch $.

{\bf case (b): $x \notin \beta \K_j$, $x \in \K$ \footnote{In the follow proof, if not mentioned specifically, every points are in $\K$}}

This part of the proof consists of three steps. First, We find a point $x_j$ in center of $\K_j$ that has low $F^j$ value. Then we find a point $x_p$ inside $\beta \K_j$, on the line between $x_j$ and $x$,  with large $\FLCE^j$ value, which implies by lemma  \ref{lem:geo_co} it has large $F^j$ value. Finally, we use both $x_0,x_p$ to deduce the large value of $F^j(x)$.

{\bf Step1:} Let $x_j$ be the center of \MVEE \ $\ellipsoid_j$ of $\convexSet_j$. By (2) in Lemma \ref{lem:geo_co}, we know that $F^j (x_j)  \le 2 \levelSetValue$. 

{\bf Step 2:} Define $H' = \{ y \mid \langle y, h_j \rangle = w_j \}$ to be the hyperplane parallel to $H $ such that $\distance(x_j, H') = \frac{1}{2}\distance(x_j, H)$, and $H'' =  \{ y \mid \langle y, h_j \rangle = u_j \}$ to be the hyperplane parallel to $H$ such that $\distance(x_0, H'') = 9\distance(x_0, H)$.

We can assume $ \langle x_0, h_j \rangle < w_j $ ($x_0, H$ are in different side of $H'$), since we know that $\FLCE^j(x_0) \leq \frac{\ell}{4}$ by definition, and the hyperplane $H'$ separates such that all points with $\langle x_0, h_j \rangle \geq   w_j $ (See $\text{ShrinkSet}$ for definition of $H$, $H'$) have value $\FLCE^j(x) \geq \ell$. 

Note $ \langle x_0, h_j \rangle < w_j $ implies  $\distance(x_0, H) \ge \frac{1}{2} \distance(x_j, H) = \distance(H, H')$ \footnote{$H, H' , H''$ are parallel to each other, so we can define distance between them}, which implies that 
$$\distance(x_j, H'') \ge  \distance(H, H'') = 8 \distance(x_0, H) \ge 4 \distance(x_j, H).$$

Now, let $x_s = l \cap H''$ be the intersection point between $H''$ and $l$, we can get: $x_s  = \xint + 8(\xint - x_0)$. Since $x_0 , \xint \in \K_j$ , we can obtain $x_s \in \frac{\centerShrink }{2}\convexSet_j$ by our choice of $\centerShrink \ge 64 d^2$. Let $x_s' = l' \cap H''$ be the intersection point of  $H''$ and the line segment $l'$ of $x$ and $x_j$. Let $x_1$ be the intersection point of $H'$ and $l$: $\{x_1\} = H' \cap l$.

Consider the plane defined by $x_0,x_j,x$. Define $x_p$ to be the intersecting point of the ray shooting from $x_s$ towards the interval $[x,x_j]$, that is parallel to the line from $x_1$ to $x_j$. 

Note that $\|x_s - x_p\| \leq \|x_1 - x_j\|$, we have: $$x_p = x_s + (x_p - x_s) = x_s + (x_j - x_1) \frac{||x_s - x_p||}{||x_1 - x_j||} $$

We know that $x_1, x_j \in \convexSet_j$, $x_s \in \frac{\centerShrink }{2}\convexSet_j$, therefore, $ x_s + (x_j - x_1) \frac{||x_s - x_p||}{||x_1 - x_j||}  \in  \centerShrink \convexSet_j$, which means $x_p \in \centerShrink \convexSet_j$. Moreover, we know that $||x_s' - x_p ||_2 \le || x_p - \xint'||_2$ due to the fact that $||x_s' - x_p ||_2 \le  || \xint' - x_j||_2 \le \frac{1}{2} \| \xint' - x_s' \|_2$ (last inequality by $\distance(x_j, H'') \ge 4 \distance(x_j, H)$).

We also note that $||x_s' - x_p ||_2 \le || x_p - \xint'||_2$ implies 
$$\textsf{dist}(x_p, H) \ge \frac{1}{2} \textsf{dist}(x_s', H).$$

Now, let $l''$ be the line segment between $x_p$ and $x_0$, let $\xint''$ be the intersection point of  $H$ and $l''$: $H \cap l'' = \{\xint''\}$.

Consider the value of $F^j(x_p)$, by (1) in Lemma \ref{lem:geo_co} and $x_p \in \centerShrink \set{K}_j$, we know that $F^j(x_p) \ge \FLCE^j(x_p)$. By the convexity of $\FLCE^j$, we obtain:
\begin{eqnarray*}
\frac{\FLCE^j(x_p) - \FLCE^j(\xint'')}{\FLCE^j(\xint'') - \FLCE^j(x_0)} &\ge& \frac{||x_p - \xint''||_2}{||\xint'' - x_0||_2} 
\\
&=& \frac{\textsf{dist}(x_p, H)}{\textsf{dist}(x_0, H)} 
\\
&\ge& \frac{\frac{1}{2}\textsf{dist}(x_s', H)}{\textsf{dist}(x_0, H)}
\\
&=& \frac{\frac{1}{2}\textsf{dist}(H'', H)}{\textsf{dist}(x_0, H)} = 4
\end{eqnarray*}

Note that $\FLCE^j(\xint'') \ge \levelSetValue$, $\FLCE^j(x_0) \le \frac{\levelSetValue}{4}$, therefore, 
$ \FLCE^j(x_p) \ge 3 \levelSetValue$. Which implies $F^j(x_p) \ge\FLCE^j(x_p) \ge 3 \levelSetValue$.

{\bf Step 3:} 

Due to $x \notin \beta \K_j$ and $x_m \in \K_j$, by our choice of $x_s$ and $\beta$, we know that $||x - \xint||_2 \ge 8 ||x_s - \xint ||_2$.

We ready to bound the value of $F^j (x)$: By the convexity of $F^j $, we have:
\begin{eqnarray*}
\frac{F^j(x) - F^j(x_p)}{F^j(x_p) - F^j(x_j)} &\ge& \frac{||x - x_p||_2}{||x_p - x_j||_2}  = \frac{||x - x_s||_2}{||x_s - x_1||_2}  \quad \quad \text{triangle similarity}
\\
&=& \frac{||x - \xint||_2 - ||x_s - \xint ||_2}{||\xint - x_1||_2 +||x_s - \xint ||_2}
\\
&\ge& \frac{||x - \xint||_2}{ 2 ||x_s - \xint ||_2} \quad \quad  \mbox{ by $||x_s - \xint ||_2 \ge 8 || \xint - x_1 ||$}
\\
&\ge&\frac{||x - \xint||_2}{||\xint - x_0||_2 } \times \frac{||\xint - x_0||_2 }{2||x_s - \xint ||_2}
\\
&\ge& \frac{\distanceRatio(x, \convexSet_{\epo}) - 1 }{32d} 
\end{eqnarray*}

The last inequality is due to $ \frac{||\xint - x_0||_2 }{||x_s - \xint ||_2} = \frac{1}{8}$ and $\frac{|| x - \xint||_2} {||\xint - x_{0}||_2}\ge   \frac{\distanceRatio(x, \convexSet_{\epo})  - 1}{2d}$

Putting together, we obtain (by $F^j(x_j)  \le 2 \levelSetValue$):
$$F^j(x) \ge  \frac{\distanceRatio(x, \convexSet_{\epo})  - 1 }{32d}  \left( F^j(x_p) - F^j(x_j) \right) \ge \frac{\distanceRatio(x, \convexSet_{\epo})  - 1}{32d} $$

Same as {\bf case (a) }, we can sum over rest epoch to obtain: 
$$\sum_{i=0}^{\epo - 1} F^i (x)  \ge \frac{(\distanceRatio(x, \convexSet_{\epo})  - 1) \levelSetValue}{32d}  - \epo  \times  \frac{4d \distanceRatio(x, \convexSet_{\epo}) \levelSetValue}{\twoExtendRatio}\ge \frac{\distanceRatio(x, \convexSet_{\epo}) \levelSetValue}{64d}$$

by our choice of parameters $\epo \le  \numEpoch$ and $ \extendRatio = \extendnumEpoch $.

\end{proof}

\subsection{Proof of Lemma \ref{lem:restart}}
\begin{proof}[Proof of Lemma \ref{lem:restart}]

Suppose algorithm {RESTART} at epoch $\epo$, then $\sum_{j \le \epo} v_j (\set{A}) \le \frac{\levelSetValue}{128 d}$. Therefore, we just need to show that for every $x \in \set{K}$, $$\sum_{i \in \Epo_0 \cup ... \cup \Epo_{\epo}} f_i(x) \ge \frac{\levelSetValue}{128d}$$. 

(a). Since the algorithm RESTART, by the RESTART condition, for every $x \in \convexSet_{\epo}$, we know that $\exists j \le \epo$ such that $F^j(x) = \sum_{i \in \Epo_j} f_i(x) \ge \FLCE^j (x)> \frac{\levelSetValue}{4}$. Using Lemma $\ref{lem:inside}$, we know that for every $j' \le \epo, j' \not= j$: $F^{j'} (x)= \sum_{i \in \Epo_{j'}} f_i (x) \ge - \frac{2 \levelSetValue}{\twoExtendRatio}$.

Which implies that $$\sum_{i \in \Epo_0 \cup ... \cup \Epo_{\epo}} f_i(x) \ge \frac{\levelSetValue}{4} - \epo \frac{2\levelSetValue}{\twoExtendRatio} \ge  \frac{\levelSetValue}{8} $$

(b). For every $x \notin \convexSet_{\epo}$, by Lemma \ref{lem:begin}, we know that 
$$\sum_{i \in  \Epo_{0} \cup ... \cup \Epo_{\epo - 1}} f_i (x) \ge \frac{\distanceRatio(x, \convexSet_{\epo}) \levelSetValue}{64d}$$

Moreover, by Lemma $\ref{lem:inside}$, we know that 
$$\sum_{i \in  \Epo_{\epo}} f_i (x) \ge  - \frac{2\distanceRatio(x, \convexSet_{\epo}) \levelSetValue}{\twoExtendRatio}$$

Putting together we have:
$$\sum_{i \in \Epo_0 \cup ... \cup \Epo_{\epo}} f_i(x) \ge \distanceRatio(x, \convexSet_{\epo})  \left( \frac{\levelSetValue}{64d} -  \frac{2\levelSetValue}{\twoExtendRatio} \right) \ge  \frac{\levelSetValue}{128d} $$

\end{proof}

\section{The EXP3 algorithm}  \label{sec:exp3}

For completeness, we give in this section the definition of the EXP3.P algorithm of \cite{Auer2003}, in slight modification which allows for unknown time horizon and output of the variances. 

\begin{algorithm}[h!]
\caption{Exp3.P}
\begin{algorithmic}[1]
\STATE Initial: $T = 1$.
\STATE Input: $K$ experts, unknown rounds. In round $t$ the cost function is given by $f_t$.
\STATE Let $\gamma = \sqrt{\frac{K \ln K}{T}}$, $\alpha = \sqrt{\ln\left( \frac{KT}{\delta} \right)}$,
\FOR{$j = 1, ..., K$}
\STATE set $$w_1(j) = \exp \left( \blowUp  \alpha \gamma\sqrt{\frac{T}{K}} \right)$$
\ENDFOR
\FOR{$t = T, ..., 2T - 1$}
 \FOR{$j = 1, ..., K$}
 \STATE $$p_t(j) = (1 - \gamma) \frac{w_t(j)}{\sum_{j'} w_{t}(j')} + \frac{\gamma}{K}$$
 \ENDFOR
 \STATE pick $j_t$ at random according to $p_t(j)$, play expert $j_t$ and receive $f_t(j_t)$
 \FOR{$j = 1, ..., K$}
  \STATE Let $$\hat{f}_t(j) = \left\{ \begin{array}{ll}
         \frac{f_t(j)}{p_t(j)} & \mbox{if $j = j_t$};\\
       0& \mbox{otherwise}.\end{array} \right.$$
          \STATE And $$\hat{g}_t(j) = \left\{ \begin{array}{ll}
         \frac{1 - f_t(j)}{p_t(j)} & \mbox{if $j = j_t$};\\
       0& \mbox{otherwise}.\end{array} \right.$$
 \ENDFOR

 \STATE Update $$w_{t + 1}(j) = w_t(j) \exp \left(  \frac{\gamma}{K} \left(\hat{g}_t(j) + \frac{ \blowUp \alpha}{p_t(j)\sqrt{TK}} \right)\right)$$
 \RETURN    $$v_t(j) = \sum_{i = 1}^t \hat{f}_i(j)$$ and $$\sigma_t(j) =\sum_{i = 1}^t \frac{  \alpha}{p_i(j)\sqrt{TK}} $$
\ENDFOR
\STATE Set $T = 2T$ and Goto 3.
\end{algorithmic}
\label{lag:black_box_1}
\end{algorithm}

\section{Acknowledgements}

We would like to thank Aleksander Madry for very helpful conversations during the early stages of this work.

\newpage

\bibliography{math}
\bibliographystyle{plain}

\end{document}

%% file: defs.tex
\newtheorem{thm}{Theorem}[section]
\newtheorem{lem}{Lemma}[section]
\newtheorem{claim}{Claim}[section]
\newtheorem{defn}{Definition}[section]

\newtheorem{cor}{Corollary}[section]

\newcommand{\set}[1]{\mathcal{#1}}

\newtheorem{theorem}{Theorem}[section]

\newtheorem{lemma}[theorem]{Lemma}

\def\reals{{\mathbb R}}
\def\R{{\mathcal R}}

\DeclareMathOperator{\grid}{grid}

\newcommand{\ellipsoid}{{\mathcal E}}

\newcommand{\conv}{\mathop{\mbox{\rm conv}}}

\newcommand{\poly}{\mathop{\mbox{\rm poly}}}

\newcommand{\ignore}[1]{}

\def\reals{{\mathbb R}}

\newcommand\ball{\mathbb{B}}

\def\bold0{\mathbf{0}}

\newcommand\mycases[4] {{
\left\{
\begin{array}{ll}
    {#1} & {#2} \\\\
    {#3} & {#4}
\end{array}
\right. }}

\def\vol{\mbox{\rm vol}}

\newcommand{\K}{\ensuremath{\mathcal K}}

\def\regret{\ensuremath{\mbox{Regret}}}

\def\epsilon{\varepsilon}

\def\R{\ensuremath{\mathcal R}}

\newcommand{\distance}{\textsf{dist}}

\newcommand{\Real}{\mathbb{R}}
\newcommand{\DM}{\text{DecideMove}}

\newcommand{\distanceRatio}{\gamma}

\newcommand{\tO}{\tilde{O}}
\newcommand{\tf}{\tilde{f}}
\newcommand{\tF}{\tilde{F}}

\newcommand{\cube}{\mathcal{Q}}

\newcommand{\xint}{x_{\textsf{m}}}

\newcommand{\lvalue}{\left(2^{d^4} (\log T)^{2d} \log \frac{1}{\delta} \right)\sqrt{T}}
\newcommand{\LRA}{\set{A}}
\newcommand{\vLRA}{v}
\newcommand{\varLRA}{\sigma}

\newcommand{\epo}{\tau}
\newcommand{\Epo}{\Gamma}

\newcommand{\LCE}{\textsf{LCE}}
\newcommand{\SLCE}{\textsf{SLCE}}

\newcommand{\FLCE}{F_{\LCE}}
\newcommand{\FSLCE}{F_{\SLCE}}
\newcommand{\blowUp}{\eta}

\newcommand{\centerShrink}{\beta}
\newcommand{\levelSetValue}{\ell}

\newcommand{\convexSet}{\mathcal{K}}
\newcommand{\hconvexSet}{{\convexSet}}
\newcommand{\gridScale}{\alpha}
\newcommand{\centerPoint}{x}
\newcommand{\radiusSearch}{r}
\newcommand{\extendRatio}{\gamma}
\newcommand{\twoExtendRatio}{ \gamma}
\newcommand{\numEpoch}{8d^2 \log T}

\newcommand{\valObtain}{1024 d^3 \log T}

\newcommand{\textendnumEpoch}{4096d^4 \log T}
\newcommand{\extendnumEpoch}{2048d^4 \log T}
\newcommand{\boundingBox}{\set{C}}

%% file: introduction.tex

\section{Introduction}

In the setting of Bandit Convex Optimization (BCO), a learner repeatedly chooses a point in a convex decision set.  The learner then observes a loss which is equal to the value of an adversarially chosen  convex loss function.  The only feedback available to the learner is the loss --- a single real number. Her goal is to minimize the regret, defined to be  the difference between the sum of losses incurred  and the loss of the best fixed decision (point in the decision set) in hindsight. 

This fundamental decision making setting is extremely general, and has been used to efficiently model online prediction problems with limited feedback such as online routing, online ranking and ad placement,   and many others (see \cite{BubeckSurvey} and \cite{HazanBook} chapter 6 for applications and a detailed survey of BCO).  This generality and importance is accompanied by significant difficulties: BCO allows for an adversarially chosen cost functions, and extremely limited information is available to the leaner in the form of a single scalar per iteration.  The extreme exploration-exploitation tradeoff common in bandit problems  is accompanied by the additional challenge of polynomial time convex optimization to make this problem one of the most difficult encountered in learning theory.  

As such, the setting of BCO has been extremely well studied in recent years and the state-of-the-art significantly advanced. For example, in case the adversarial cost functions are linear, efficient algorithms are known that guarantee near-optimal regret bounds \cite{Abernethy08,BubeckCK12,HazanK14}.  A host of techniques have been developed to tackle the difficulties of partial information, exploration-exploitation and efficient convex optimization. Indeed, most known optimization and algorithmic techniques have been applied, including interior point methods \cite{Abernethy08}, random walk optimization \cite{NarayananR10}, continuous multiplicative updates \cite{Dani07}, random perturbation  \cite{AwerbuchK08}, iterative optimization methods \cite{Flaxman} and many more.

Despite this impressive and the long lasting effort and progress, the main question of BCO remains unresolved: construct an efficient and optimal regret algorithm for the full setting of BCO. Even the optimal regret attainable is yet unresolved in the full adversarial setting. 

A significant breakthrough was recently made by  \cite{BubeckDKP15}, who show that in the oblivious setting and in the special case of 1-dimensional BCO, $O(\sqrt{T})$ regret is attainable. Their result is existential in nature, showing that the minimax regret for the oblivious BCO setting (in which the adversary decides upon a distribution over cost functions independently of the learners' actions) behaves as $\tilde{\Theta}(\sqrt{T})$. This result was very recently extended to any dimension by \cite{BubeckE15}, still with an existential bound rather than an explicit algorithm and in the oblivious setting. 

\medskip

In this paper we advance the state of the art in bandit convex optimization and show the following results:
\begin{enumerate}
\item
We show that  minimax regret for the full {\bf adversarial} BCO setting is $\tilde{\Theta}(\sqrt{T})$.  
\item
We give an explicit algorithm attaining this regret bound. Such an explicit algorithm was unknown previously even for the oblivious setting. 
\item
The algorithm guarantees $\tilde{\Theta}(\sqrt{T})$ regret with high probability and exponentially decaying tails. Specifically, the algorithm guarantees regret of $\tilde{\Theta}(\sqrt{T} \log \frac{1}{\delta})$ with probability at least $1-\delta$. 
\end{enumerate}

It is known that any algorithm for BCO must suffer regret $\Omega(\sqrt{T})$ in the worst case, even for oblivious adversaries and linear cost functions. Thus, up to logarithmic factors, our results close the gap of the attainable regret  in terms of the number of iterations. 

To obtain these results we introduce some new techniques into online learning, namely a novel online variant of the ellipsoid algorithm, and define some new notions in discrete convex geometry. 

\paragraph{What remains open?} 
Our algorithms depend exponentially on the dimensionality of the decision set,  both in terms of regret bounds as well as in  computational complexity. As of the time of writing, we do not know whether this dependencies are tight or can be improved to be polynomial in terms of the dimension, and we leave it as an open problem to  resolve this question\footnote{In the oblivious setting \cite{BubeckE15} show that the regret behaves polynomially in the dimension. It is not clear if this result can be extended to the adversarial setting.}.

\subsection{Prior work} 

The best known upper bound in the regret attainable for adversarial BCO with general convex loss functions is $\tilde{O}(T^{5/6})$ due to  \cite{Flaxman} and  \cite{Klien04} \footnote{although not specified precisely to the adversarial setting, this result is implicit in these works.}.  A lower bound of   $\Omega(\sqrt{T})$  is folklore, even the easier full-information setting of online convex optimization, see e.g. \cite{HazanBook}.

The special case of bandit linear optimization (BCO in case where the adversary is limited to using linear losses) is significantly simper. Informally, this is since the average of the function value on a sphere around a center point equals the value of the function in the center, regardless of how large is the sphere. This allows for very efficient exploration, and was first used by \cite{Dani07} to devise the  Geometric Hedge algorithm that achieves an optimal regret  rate of $\tO(\sqrt{T})$. An efficient algorithm inspired by interior point methods was later given by \cite{Abernethy08}  with the same optimal regret bound. Further improvements in terms of the dimension and other constants were subsequently given in \cite{BubeckCK12,HazanK14}.

The first gradient-descent-based method for BCO was given by \cite{Flaxman}. Their regret bound was subsequently improved for various special cases of loss functions using ideas from \cite{Abernethy08}.  For  convex and smooth losses, \cite{Saha11} attained an upper bound on the regret of of $\tO(T^{2/3})$. This was recently improved to  by \cite{DekelEK15} to $\tilde{O}(T^{5/8})$.  \cite{Ofer10} obtained a regret bound of $\tO(T^{2/3})$ for  strongly-convex losses.  
 For the special case of   strongly-convex and smooth losses, \cite{Ofer10}  obtained a regret of $\tO(\sqrt{T})$ in the unconstrained case, and \cite{HazanL14} obtain the same rate even in the constrained cased.  \cite{Shamir} gives a  lower bound of $\Omega(\sqrt{T})$  for the setting of strongly-convex and smooth BCO. 
 
A comprehensive survey by Bubeck and Cesa-Bianchi \cite{BubeckSurvey}, provides a review of the bandit optimization literature in both stochastic and online setting.

Another very relevant line of work is that on zero-order convex optimization. This is the  setting of convex optimization in which the only information available to the optimizer is a valuation oracle that given $x \in \K$ for some convex set $\K \subseteq \reals^d$,  returns $f(x)$ for some convex function $f: \K \mapsto \reals$ (or a noisy estimate of this number). This is considered one of the hardest areas in convex optimization (although strictly a special case of BCO), and a significant body of work has culminated in a polynomial time algorithm, see \cite{conn2009introduction}. Recently, \cite{Agarwal13}  give a polynomial time algorithm for regret minimization in the stochastic setting of zero-order optimization, greatly improving upon the known running times.

\subsection{Paper structure}

In the next section we give some basic definitions and constructs that will be of use. In section \ref{sec:regression} we survey a natural approach, motivated by zero-order optimization, and explain why completely new tools are necessary to apply it. We proceed to give the new mathematical constructions for discrete convex geometry in section \ref{sec:geometry}. This is followed by our main technical lemma, the discretization lemma, in section \ref{sec:discretization}. We proceed to give the new algorithm  and the main result statement in section \ref{sec:algorithm}.

%% file: convex-regression.tex

\section{The insufficiency of convex regression}  \label{sec:regression}

Before proceeding to give the main technical contributions of this paper, we give some description of the technical difficulties that are encountered and intuition as to how they are resolved.

A natural approach for BCO, and generally for online learning, is to borrow ideas from the less general setting of stochastic zero-order optimization. Till recently, the only polynomial time algorithm for zero-order optimization was based on the ellipsoid method \cite{GLS}. Roughly speaking, the idea is to maintain a subset, usually an ellipsoid,  in space in which the minimum resides, and iteratively reduce the volume of this region till it is ultimately found.

In order to reduce the volume of the ellipsoid one has to find a hyperplane separating the minimum and a large constant fraction of the current ellipsoid in terms of volume. In the stochastic case, such a hyperplane can be found by sampling and estimating a sufficiently indicative region of space. A simple way to estimate the underlying convex function in the stochastic setting is called convex regression (although much more time and query-efficient methods are known, e.g. \cite{Agarwal13}).

Formally, given noisy observations from a convex function $f: \K \mapsto \reals^d$, denoted $\{v(x_1),...,v(x_n) \}$, such that $v(x_i)$ is a random variable whose expectation is $f(x_i)$, the problem of convex regression is to create an estimator of the value of $f$ over the entire space which is consistent, i.e. approaches its expectation as the number of observations increases $n \mapsto \infty$.  The methodology of convex regression  proceeds by solving a convex program to minimize the mean square error and ensuring convexity by adding gradient constraints, formally,
\begin{align*}
& \min \sum_{i=1}^n ( v(x_i ) - y_i ) ^2 \\
& y_j \geq y_i + \nabla_i^\top (x_j  - x_i)   
\end{align*}
In this convex program  $\{\nabla_i, y_i \}$ are variables, points $x_i$ are chosen by the algorithm designer to observe, and $v(x_i)$ the observed values from sampling. Intuitively, there are $nd + n$ degrees of freedom ($n$ scalars and $n$ vectors in $d$ dimensions) and $O(n^2)$ constraints, which ensures that this convex program has a unique solution and generates a consistent estimator for the values of $f$ w.h.p. (see \cite{convexregression} for more details). 

The natural approach of iteratively applying convex regression to find a separating hyperplane within an ellipsoid algorithm fails for BCO because of the following difficulties:
\begin{enumerate}
\item
The ellipsoid method was thus far not applied successfully in online learning, since the optimum is not fixed and can change in response to the algorithms' behavior.  Even within a particular ellipsoid, the optimal strategy is not stationary. 

\item
Estimation using convex regression over a fixed grid is insufficient, since arbitrarily deep ``valleys" can hide between the grid points. 

\end{enumerate}

Our algorithm and analysis below indeed follows the general ellipsoidal scheme, and overcomes these difficulties by:
\begin{enumerate}
\item
The ellipsoid method is applied with an optional ``restart button". If the algorithm finds that the optimum is not within the current ellipsoidal set, it restarts from scratch. We show that by the time this happens, the algorithm has accumulated so much negative regret that it only helps the player. Further, inside each ellipsoid we use the standard multiarmed bandit algorithm EXP3.P due to \cite{Auer2003}, to exploit and explore it. 

\item
A new estimation procedure is required to ensure that no valleys are missed. For this reason we develop some new machinery in convex geometry and convex regression that we call the lower convex envelope of a function.  This is a convex lower bound on the original function that ensures there are no valleys missed, and in addition needs only constant-precision grids for being consistent with the original function. 

This contribution is the most technical part of the paper, as culminates in the "discretization lemma", and can be skimmed at first read. 

\end{enumerate}